\documentclass[11pt]{article}
\usepackage{enumerate}
\usepackage{pdfsync}
\usepackage[OT1]{fontenc}

\usepackage{url}            
\usepackage{booktabs}       
\usepackage{amsfonts}       
\usepackage{nicefrac}       
\usepackage{microtype}      
\usepackage{xcolor}

\usepackage[colorlinks,
linkcolor=red,
anchorcolor=blue,
citecolor=blue
]{hyperref}

\usepackage{fullpage}
\usepackage{setspace}
\usepackage{tabularx}
\usepackage{float}
\usepackage{wrapfig,lipsum}
\usepackage{enumitem}
\usepackage{natbib}
\usepackage{multirow}

\makeatletter
\newcommand*{\rom}[1]{\expandafter\@slowromancap\romannumeral #1@}
\makeatother

\input{uclaml_macro.sty}

\newcommand{\parenv}[1]{\left( #1 \right)}
\newcommand{\ceil}[1]{\lceil #1 \rceil}

\allowdisplaybreaks
\title{Adaptive Sampling for Heterogeneous Rank Aggregation from Noisy Pairwise Comparisons}

\author{
	Yue Wu\thanks{Department of Computer Science, University of California, Los Angeles, Los Angeles, CA 90095; e-mail: {\tt wuy@g.ucla.edu}}~\thanks{Equal contribution}
    ~~~and~~~
    Tao Jin\thanks{Department of Computer Science, University of Virginia, Charlottesville, VA 22904; e-mail: {\tt taoj@virginia.edu}}~\footnotemark[2]
	~~~and~~~
	Hao Lou\thanks{Department of Electrical and Computer Engineering, University of Virginia, Charlottesville, VA 22904; e-mail: {\tt hl2nu@virginia.edu}}
	~~~and~~~
	Pan Xu\thanks{Department of Computing and Mathematical Sciences, California Institute of Technology, Pasadena, CA 91125; e-mail: {\tt panxu@caltech.edu}}
	~~~and~~~ \\
	Farzad Farnoud\thanks{Department of Electrical and Computer Engineering, Department of Computer Science, University of Virginia, Charlottesville, VA 22904; e-mail: {\tt farzad@virginia.edu}}~\thanks{Co-corresponding authors}
	~~~and~~~
	Quanquan Gu\thanks{Department of Computer Science, University of California, Los Angeles, Los Angeles, CA 90095; e-mail: {\tt qgu@cs.ucla.edu}}~~\footnotemark[7]
}
\begin{document}

\date{}
\maketitle

\begin{abstract}
In heterogeneous rank aggregation problems, users often exhibit various accuracy levels when comparing pairs of items. Thus a uniform querying strategy over users may not be optimal. To address this issue, we propose an elimination-based active sampling strategy, which estimates the ranking of items via noisy pairwise comparisons from users and improves the users' average accuracy by maintaining an active set of users. We prove that our algorithm can return the true ranking of items with high probability. We also provide a sample complexity bound for the proposed algorithm which is better than that of non-active strategies in the literature. Experiments are provided to show the empirical advantage of the proposed methods over the state-of-the-art baselines.

\end{abstract}

\section{Introduction}
To rank a set of items from noisy pairwise comparisons or preferences is a widely studied topic in machine learning \citep{braverman2008noisy, weng2011bayesian,ren2019sample,jin2020rank}. This is also referred to as rank aggregation. Multiple applications exist for this task: ranking online game players \citep{Herbrich2006TrueSkillTMAB}, evaluating agents in games \citep{Rowland2019MultiagentEU}, recommendation systems \citep{Valcarce2017CombiningTR}, etc.
In the above mentioned cases, all data used in inference shares the assumption that each preference has the same credibility. In a heterogeneous setting, the providers of fractions of data may have varying unknown accuracy levels. Thus, it is natural to take advantage of the more accurate ones to obtain a more accurate ranking using a smaller number of queries. 

Nowadays, it is common to collect large scale datasets in order to facilitate the process of knowledge discovery. Due to its large scale,  data collection is usually carried out by crowdsourcing \citep{kumar2011,chen2013pairwise}, where different parties with heterogeneous backgrounds generate a subsets of the data. While crowdsourcing makes it possible to scale up the size, it also brings in new challenges when it comes to the cost of operation and cleanness of the data. Therefore, the optimal ranking algorithm in the single user setting \citep{ren2019sample} may not be straightforwardly extended to the heterogeneous setting while maintaining optimality. In particular, if we know the best one among all users, the best we can do is to apply the Iterative-Insertion-Ranking (IIR) algorithm in \citet{ren2019sample} with the best user (i.e., with the highest accuracy). Unfortunately, in practice, the accuracy of the users are often unknown. A naive way may be to randomly select a user for the query at each time step and use the comparisons provided by this user to perform IIR. As we show in later sections, this naive method usually bears a high sample complexity. Therefore, it is an interesting and important research problem how we can adaptively select a subset of users at each time to make the queries (pairwise comparisons) in order to insert an item correctly into the ranked list. 

In this paper, we study the rank aggregation problem, where a heterogeneous set of users provide noisy pairwise comparisons for the items. We propose a novel algorithm that queries pairwise comparisons of two items from a changing active user set. Specifically, we maintain a short history of user responses for a set of comparisons. When the inferred rank of these comparisons is estimated to be true with high probability, it is then used to calculate a reward based on the recorded responses. Then a UCB-style elimination process is called to remove inaccurate users from sampling pool.
We theoretically analyze the sample complexity of the proposed algorithm, which reduces to the result in the state-of-the-art ranking algorithm \citep{ren2019sample} for a single user. We conducted experiments on a synthetic dataset, which demonstrates that our adaptive sampling algorithm based on user elimination is much more sample efficient than baseline algorithms.

\textbf{Contributions of this paper} are summarized as follows:
\begin{itemize}
	\item We propose a novel algorithm for heterogeneous rank aggregation, which uses a successive elimination subroutine to adaptively maintain a set of active  users during the ranking process.
	\item The proposed algorithm shares the same order of sample complexity as that of the oracle algorithm which has access to the optimal user and uses the state-of-the-art ranking approach in the single-user setting.
	\item 
	We conducted experiments in heterogeneous rank aggregation problems where different users have different accuracy levels and only a small fraction of users are highly accurate in providing the comparison result. Our experiment shows that the proposed algorithm costs significantly fewer samples than baseline algorithms in order to recover the exact ranking. 
\end{itemize}

The remainder of this paper is organized as follows: we review the most relevant work in the literature to ours in Section \ref{sec:related}. We present some preliminaries of ranking from noisy pairwise comparisons in Section \ref{sec:preliminary} and discuss the challenge of extending single-user optimal ranking algorithms to the heterogeneous setting. In Section \ref{sec:alg_elim}, we present our main algorithm and provide detailed description of the method. Then in Section \ref{sec:main_theory}, we provide theoretical analysis of the upper bound on the sample complexity of the proposed algorithm and compare the results with baseline algorithms. We conduct numerical experiments in Section \ref{sec:experiment}  to demonstrate the empirical superiority of our method. Finally, we conclude the paper with Section \ref{sec:conclusion}.

\noindent\textbf{Notation} We use lower case letters to denote scalars, and lower and upper case bold letters to denote vectors and matrices. We use $\| \cdot \|$ to indicate the Euclidean norm. We also use the standard $O$ and $\Omega$ notations. We say $a_n = O(b_n)$ if and only if $\exists C > 0, N > 0, \forall n > N, a_n \le C b_n$; $a_n = \Omega(b_n)$ if $a_n \ge C b_n$; $a_n = \Theta(b_n)$ if and only if $a_n = O(b_n)$ and  $a_n = \Omega(b_n)$. The notations like $\tilde{O}$ are used to hide logarithmic factors. For a positive integer $N$, $[N] := \{1,2,\dots,N\}$.

\section{Related Work}\label{sec:related}
In this section, we discuss two closely related topics to our work, which cover the two aspect of heterogeneous rank aggregation: active ranking to infer rank and best arm identification to select an accurate subset of information sources. 

\textbf{Active ranking. }
For non-active ranking problems, usually a static dataset is given at prior, or the algorithms only start to infer the model parameters when all the comparison is collected. Some common models are the BTL model \citep{bradley1952} and the Thurstone model \citep{thurstone1927}. However, with the help of assumptions derived from specific models, it is possible to infer certain relationship among items with high probability of correctness without explicitly querying the others.
In most active ranking from noisy comparisons algorithms, there are components for online estimation of the rank and identification of the most informative pairs to be queried. 
For instance, in \cite{maystre2017just}, with the assumption that the true scores for $n$ items are generated by a Poisson process, with a fixed number of $O(n \text{poly}(\log(n))$ comparisons, an \emph{approximate} ranking of $n$ items can be found. 
Let $\Delta_i$ be the instance specific parameters characterized by the distances between items and $\Delta_{\min} = \min_{i \in [n]} \Delta_i$. An instance-aware sample complexity bound of $O(n\log(n)\Delta_{\min}^{-2}\log(n/(\delta\Delta_{\min}))$ is provided along with a QuickSort based algorithm by \cite{szorenyi2015online}. 
In \cite{ren2019sample}, an analysis for a distribution agnostic active ranking scheme is provided. To achieve a $\delta$-correct \emph{exact} ranking, $O(\sum_{i \in [n]}\Delta_i^{-2} (\log\log(\Delta_i^{-1}) + \log(n/\delta)))$ comparisons are required. The exact inference requirement results in repetitive queries of the same pair consecutively which costs a constant overhead compared to \emph{approximate} inference.

\textbf{Best arm identification} is a pure exploration method in multi-armed bandits   \citep{audibert2010best,chen2017adaptive}.
In the crowdsourcing setting, every user can be queried with the same question. Suppose some users can provide more accurate answers compared to others. Then it comes to the question of how to identify the best user. We can regard the choice of which user to ask as an action. And the correctness of the user's response as the reward (cost) of the taken action. A long line of research has explored the identification of the best action facing stochastic feedback. Recently, \cite{resler2019noise} studied cases when the observed binary action costs can be inverted with a probability which is less than half. With a careful construct of the estimated cost despite the noise, the regret of the online algorithm suffers a constant order compared to the noiseless setting even without the knowledge of the inversion probability. 


		
\section{Preliminaries and Problem Setup}\label{sec:preliminary}
\subsection{Ranking From Noisy Pairwise Comparisons}
Suppose there are $N$ items in total that we want to rank and $M$ users to be queried. 
An item is indexed by an integer $ I \in [N]$.  We assume there is a unique \emph{true ranking} among the $N$ items. A user is also indexed by an integer $u \in [M]$. For a subset of users, we use $\cU \subseteq [M]$ to denote the index set. Every time step, we can pick a pair of items $i$ and $j$ and ask a user whether item $i$ is better than item $j$. The comparison returned by the user may be noisy. In many applications, we have more than one users that we can query. We assume that for any pair of items $(i,j)$ with true ranking $i \succ j$, the probability that user $u$ answers the query correctly is $p_u(i,j) = \Delta_u+\frac{1}{2}$, where $\Delta_u \in (0, \frac{1}{2}]$ is referred to as the accuracy of user $u$. When some of the $\Delta_u$ are different from the others, we have a heterogeneous set of users. 
Formally, this paper aims to solve the following problem:
\begin{definition}[Exact Ranking with Multiple Users]
Given $N$ items, $M$ users and $\delta \in (0,1)$, we want to identify the true ranking among the $N$ items with probability at least $1-\delta$. An algorithm $\cA$ is $\delta$-correct if, for any instance of the input, it will return the correct result in finite time with a probability at least $1-\delta$.
\end{definition}

\begin{definition}
Let $\cU\subseteq [M]$ be an arbitrary subset of users. If a user in $\cU$, denoted as $x$, satisfies $\Delta_x+\alpha\ge \max_{u \in \cU}\Delta_u$, then it is called an $\alpha$-optimal user in $\cU$. If a user is $\alpha$-optimal among all $M$ users, then it is called an (global) $\alpha$-optimal user.
\end{definition}


\subsection{Iterative Insertion Ranking With a Single User}
When there is only one user $u$ to be queried ($M=1$), the above problem reduces to the exact rank problem for which \citet{ren2019sample} proposed the Iterative-Insertion-Ranking (IIR) algorithm. The sample complexity (i.e., the total number of queries) to achieve exact ranking with probability $1-\delta$ is characterized by the following proposition:
\begin{proposition}[Adapted from Theorems 2 and 12, \citet{ren2019sample}]\label{prop:comIIR}
Given $\delta \in (0, 1/12)$ and an instance of $N$ items, the number of comparisons used by any $\delta$-correct algorithm $\cA$ 
on this instance is at least
\begin{align}\label{eq:lower_bd_single_user}
    \Theta
    \big(
    N \Delta_{u}^{-2}
    \big(
    \log \log \Delta_{u}^{-1}
    +
    \log(N/\delta)
    \big)
    \big),
\end{align}
and the Iterative-Insertion-Ranking algorithm (presented in \citet{ren2019sample}) can indeed output the exact ranking using number of comparisons, with probability $1-\delta$. 
\end{proposition}

The complexity above can be decomposed into the complexity of inserting each item into a constructed sorting tree.

In this paper, we consider a more challenging ranking problem, where multiple users with heterogeneous levels of accuracy can be queried each time. In the multi-user setting, the optimal sample complexity in \eqref{eq:lower_bd_single_user} can be achieved only if we know which user is the best user $u^* = \arg \max_{u \in [M]} \Delta_u$. The optimal sample complexity can then be written as 
\begin{align} \label{eqn:comp-best-user}
    \cC_{u^*}(N)
    & = 
    \Theta \big(
    N \Delta_{u^*}^{-2}
    \big(
    \log \log \Delta_{u^*}^{-1}
    +
    \log(N/\delta)
    \big)
    \big).
\end{align}

However, with no prior information on the users' comparison accuracy, it is unclear whether we can achieve a sample complexity close to~\eqref{eqn:comp-best-user}. In this scenario, the most primitive route is to perform no inference on the users' accuracy, randomly choose a user and make the query. 
This leads to an equivalent accuracy margin $\Bar{\Delta}_0 = \frac{1}{M}\sum_{u \in [M]} \Delta_{u}$:
\begin{align} \label{eqn:comp-aver-user}
    \cC_{\mathrm{ave}}(N)
    & = 
    \Theta \big(
    N \Bar{\Delta}_0^{-2}
    \big(
    \log \log \Bar{\Delta}_0^{-1}
    +
    \log(N/\delta)
    \big)
    \big),
\end{align}
which clearly has a gap linear (ignoring logarithmic factors) in the number of items when compared with the optimal complexity. 
This is certainly undesirable, especially when there are vastly many items to be compared.
Therefore, an immediate question is:
\begin{quote}
    Can we design an algorithm that has a sublinear gap in sample complexity compared with the optimal sample complexity?
\end{quote}
What we will propose in the following section is an algorithm that can achieve a sublinear regret, where the regret is defined as the difference between the sample complexity of the proposed algorithm and  the optimal sample complexity.


\section{Adaptive Sampling and User Elimination}\label{sec:alg_elim}
The main framework of our procedure is derived based on the \textsc{Iterative-Insertion-Ranking} algorithm proposed in \citet{ren2019sample}, which, to the best of our knowledge, is the first algorithm that has matching instance-dependent upper and lower sample complexity bounds for active ranking problems in the single-user setting. Note that the strong stochastic transitivity (SST) assumption defined in \citep{falahatgar2017maximum,falahatgar2018limits} holds in our 
setting. The ranking algorithm comprises the following four hierarchical parts and operates on a Preference Interval Tree (PIT) \citep{Feige1994ComputingWN,ren2019sample}, which stores the currently inserted and sorted items (for its specific definition, please refer to Appendix \ref{sec:omitted_subrountine}):
\begin{enumerate}[leftmargin=*]
    \item \textsc{Iterative-Insertion-Ranking} (IIR): the main procedure which calls IAI to insert each item into a PIT with a high probability of correctness. It is displayed in Algorithm \ref{alg:iir}.
    \item \textsc{Iterative-Attempting-Insertion} (IAI): the subroutine which calls ATI to insert the current item $z\in[N]$ into the ranked list with a deviation error $\epsilon$, and iteratively calls ATI by decreasing the error until the probability that item $z$ is inserted to the correct position is high enough. It is displayed in Algorithm \ref{alg:iai}. 
    \item \textsc{Attempting-Insertion} (ATI): the subroutine that traverses the Preference Interval Tree using binary search \citep{feige1994computing} to find the node where the item should be inserted with error $\epsilon$. While it compares the current item and any node in the tree, it calls ATC to obtain the comparison result.  It is displayed in Algorithm \ref{alg:ati}.
    \item \textsc{Attempting-Comparison} (ATC): the subroutine that adaptively samples queries from a subset of users for a pair of items $(z,j)$, where $z$ is the item currently being inserted and $j$ is any other item. ATC records the number of queries each user provides and the results of the comparison. It is displayed in Algorithm \ref{alg:atc}. 
\end{enumerate}

In the heterogeneous rank aggregation problem, each user may have a different accuracy from the others. Therefore, we adaptively sample the comparison data from a subset of users. In particular, we maintain an active set $\cU\subseteq[M]$ of users, which contains the potentially most accurate users from the entire group. We add a user elimination phase to the main procedure (Algorithm \ref{alg:iir}) based on the elimination idea in multi-armed bandits \citep{slivkins2019introduction,lattimore2020bandit} to update this active set.
In particular, we view each user as an arm in a multi-arm bandit, where the reward is $1$ if the answer from a certain user is correct and $0$ if wrong. After an item is successfully inserted by IAI, we call Algorithm \ref{alg:elim-user} (\textsc{EliminateUser}) to eliminate users with low accuracy levels before we proceed to the next item. 

To estimate the accuracy of users, a vector $\sbb_z\in\RR^{M}$ recording the counts of responses from each user for item $z$ is maintained during the whole period of inserting item $z$.
We further keep track of two matrices $A_z,B_z\in\RR^{N\times M}$. When a pair $(z,j)$ (where $z$ refers to the item  currently being inserted and $j$ to an arbitrary item) is compared by user $u\in\RR^M$ in Algorithm \ref{alg:atc}, we increase $A[j,u]$ by $1$ if user $u$ thinks $z$ is better than $j$ and increase $B[j,u]$ by $1$ otherwise. 
We use $w$ to record the total number of times that item $z$ is deemed better by any users and use the average $\hat p=w/t$ to provide an estimation of the average accuracy $|\cU|^{-1}\sum_{u \in \cU} p_{ij}^u$. The variables $A_z, B_z$, and $\sbb_z$ are global variables, shared by different subroutines throughout the process. After an item $z$ is successfully inserted, $A_z, B_z$ can be reset to zero and the space allocated can be used for $A_{z+1}, B_{z+1}$ (See Line \ref{algline:iir_reset_global} of Algorithm \ref{alg:iir}). 

We use the $0/1$ reward for each user to indicate whether the provided pairwise comparison is correct.
Nevertheless, this reward is not known immediately after each arm-pull since the correctness depends on the ranking of items which is also unknown. But when IAI returns \textit{inserted}, the item recently inserted has a high probability to be in the right place. Our method takes advantage of this fact by constructing a ground truth pairwise comparison of this item with all other already inserted items in the PIT. Then an estimate of the reward $\bn_z$ can be obtained with the help of recorded responses $A_z$ and $B_z$, which are updated in ATC as described in the above paragraph. At last, in Algorithm \ref{alg:elim-user} a UCB-style condition is imposed on estimated accuracy levels  $\bmu = \bn_z / \bs_z$.
Due to the space limit, we omit the IAI and ATI routines since they are the same as that in \citet{ren2019sample}. We include them for the completeness of our paper in Appendix \ref{sec:omitted_subrountine}.

\begin{algorithm}
\caption{Main Procedure: \textsc{Iterative-Insertion-Ranking} (IIR)\label{alg:iir}}
\textbf{Global Variables:} \\
$z\in\NN$: the index of the item being inserted into the ranked list.\\
$A_z\in\RR^{N\times M}$: $A_z[j,u]$ is the number of times that user $u$ thinks item $z$ is better than item $j$.\\
$B_z\in\RR^{N\times M}$: $B_z[j,u]$ is the number of times that user $u$ thinks item $z$ is worse than item $j$.\\
$\sbb_z\in\RR^{M}$: total number of responses by each user so far.\\
    \textbf{Input parameters}: Items to rank $S = [N]$ and confidence $\delta$ \\
    \textbf{Initialize}: $\bn_1 = \bs_1 = \mathbf{0}$
	\begin{algorithmic}[1]
		\STATE $Ans\gets $ the list containing only $S[1]$
		\FOR{$z \gets$ $2$ to $|S|$}
		    \STATE $\bn_{z} = \bn_{z-1}, \bs_{z} = \bs_{z-1}, A_{z} = \mathbf{0}, B_z = \mathbf{0}$\label{algline:iir_reset_global}
		    \STATE IAI$(S[z],Ans,\delta/(n-1))$  \hfill $\triangleright$Algorithm \ref{alg:iai} \textcolor{gray}{(global variables $A_z, B_z, \bs_z$ are updated here)} \label{line:IIR-4}
		    \FOR {$j \in [z-1]$}
			    \IF {$S[z] > S[j]$ in PIT}
	    		    \STATE $\bn_z = \bn_{z} + A_z[j,*]$
			    \ELSE
    			    \STATE $\bn_z = \bn_{z} + B_z[j,*]$
			    \ENDIF
			\ENDFOR
		    \STATE $\cU_z \leftarrow \textsc{EliminateUser}(\cU_{z-1}, \nbb_z, \sbb_z, \delta/(n-1))$ \hfill $\triangleright$Algorithm \ref{alg:elim-user}
		\ENDFOR
		\RETURN{$Ans$}; 
	\end{algorithmic}
\end{algorithm}

\begin{algorithm}
\caption{Subroutine: \textsc{Attempt-To-Compare (ATC)} $(z, j, \cU, \epsilon, \delta)$}\label{alg:atc}
\textbf{Input:} items $(z, j)$ to be compared, set of users $\cU$, confidence parameter $\epsilon, \delta$. $M$ is the number of users originally. 
\begin{algorithmic}[1]
\STATE $m = |\cU|, \hat{p} = 0, w = 0, \hat{y} = 1$. Number of rounds $r = 1$. $r_{\max} = \lceil \frac{1}{2}\epsilon^{-2}\log\frac{2}{\delta} \rceil$.
        \WHILE{$r \leq r_{\max}$}
            \STATE Choose $u$ uniformly at random from $\cU$
            \STATE Obtain comparison result from user $u$ as $y_{ij}^u$
            \STATE Increment the counter of responses collected from this user $\bs_z[u] \leftarrow \bs_z[u] + 1$
            \IF {$y_{ij}^u > 0$}
                \STATE $A_z[j, u] \leftarrow A_z[j, u] + 1$, $w \leftarrow w + 1$
            \ELSE 
                \STATE $B_z[j, u] \leftarrow B_z[j, u] + 1$
            \ENDIF
        \STATE $\hat{p} \leftarrow w / r$, $r \leftarrow r + 1$, 
        $c_r \leftarrow \sqrt{\frac{1}{2t}\log(\frac{\pi^2r^2}{3\delta})}$ \label{code:average-response}
        \IF {$|\hat{p} - \frac{1}{2}| \geq c_r$}
            \STATE \textbf{break}
        \ENDIF
        \ENDWHILE
        \IF{$\hat{p} \leq \frac{1}{2}$}
            \STATE $\hat{y} = 0$
        \ENDIF
        \STATE \textbf{return:} $\hat{y}$
    \end{algorithmic}
\end{algorithm}

\begin{algorithm}
\caption{Subroutine: \textsc{EliminateUser}}\label{alg:elim-user}
    \textbf{Input parameters}: $(\cU, \nbb, \sbb, \delta)$ 
    \begin{algorithmic}[1]
        \STATE Set $S = \sum_{u \in [M]} \sbb_u$, $\sbb_{\min} = \min_{u \in \cU} \sbb_u$, $\bmu_u = \nbb_u / \sbb_u$, $r =  \sqrt{\frac{\log(2|\cU|/\delta)}{2 \sbb_{\min}}}$ 
        \STATE Set
        $\textbf{LCB} = \bmu - r \mathbf{1} $ and  $\textbf{UCB} = \bmu + r \mathbf{1} $.
        \IF{$S \ge 2M^2 \log(NM/\delta)$}
            \FOR{$u \in \cU$}
                \STATE Remove user $u$ from $\cU$ if  $ \exists u' \in \cU, \textbf{UCB}_u < \textbf{LCB}_{u'}$.
            \ENDFOR
        \ENDIF
        \RETURN $\cU$
    \end{algorithmic}
\end{algorithm}

\section{Theoretical Analysis}\label{sec:main_theory}

In this section, we analyze the sample complexity of the proposed algorithm and compare it with other baselines mentioned in Section \ref{sec:preliminary}.

\subsection{Sample Complexity of Algorithm \ref{alg:iir}}
We first present a guarantee on upper bound of the sample complexity of the proposed algorithm. Define $\bar{\Delta}_z = \frac{1}{\cU_z} \sum_{u \in \cU_z} \Delta_u$ to be the average accuracy of all users in the current active set. Denote 
\begin{align}\label{def:F_function}
    F(x) = x^{-2}(\log \log x^{-1} + \log(N/\delta)).
\end{align}
Although $F(x)$ depends on $N$ and $\delta^{-1}$, the dependence is only logarithmic, and it does not affect the validity of reasoning via big-$O$ notations. 
\begin{theorem} \label{theorem:q-complexity} 
For any $\delta>0$, with probability at least $1-\delta$,  Algorithm~\ref{alg:iir} returns the exact ranking of the $N$ items, and it makes at most $\cC_{\mathrm{Alg}}(N)$  queries, where
\begin{align*}
    \cC_{\mathrm{Alg}}(N)
    & = 
    O\Bigg(
    \sum_{z=2}^{N}
    \bar{\Delta}_z^{-2}  (\log\log\bar{\Delta}_z^{-1} + \log(N/\delta)) 
    \Bigg)
    =
    O\Bigg(
    \sum_{z=2}^{N}
    F(\bar{\Delta}_z)
    \Bigg).
\end{align*}
\end{theorem}
\begin{proof}
The analysis on the sample complexity  follows a similar route as \citet{ren2019sample} due to the similarity in algorithm design. In fact, since we randomly choose a user from $\cU_t$ and query it for a feedback, it is equivalent to query a single user with the averaged accuracy $\frac{1}{2} + \bar{\Delta}_z$, where $\bar{\Delta}_z := \frac{1}{|\cU_z|} \sum_{u \in \cU_z} \Delta_u$. This means most of the theoretical results from \citet{ren2019sample} can also apply to our algorithm. In Appendix~\ref{subsec:proof-complexity}, we presented more detailed reasoning. 
\end{proof}

\subsection{Comparison of Sample Complexities Among Different Algorithms}
While Theorem~\ref{theorem:q-complexity} characterizes the sample complexity of Algorithm \ref{alg:iir} explicitly, the result therein is not directly comparable with the sample complexity of the oracle algorithm that only queries the best user $\cC_{u^*}(N)$ or the complexity of the naive random-query algorithm $\cC_{\mathrm{ave}}(N)$. Based on Theorem~\ref{theorem:q-complexity}, we can derive the following more elaborate sample complexity  of Algorithm \ref{alg:iir}.
\begin{theorem} \label{theorem:complexity-gap}
Suppose there are $N$ items and $M$ users initially. 
Denote $S_z = \sum_{u \in [M]} (\sbb_z)_u$ to be the number of all queries made before inserting item $z$ (Line~\ref{line:IIR-4} in Algorithm~\ref{alg:iir}).
The proposed algorithm has the following sample complexity upper bound:
\begin{align}
    \cC_{\mathrm{Alg}}(N,M)
    & = 
    O( N F(\Delta_{u^*})
    ) +
    O
    \Bigg(
    \sum_{z=2}^{N}
    \ind 
    \big\{
    S_z < 
    2M^2 \log(NM/\delta)
    \big\}
    \big (F(\bar{\Delta}_0)
    -
    F(\Delta_{u^*}) \big)
    \Bigg)
    \notag \\ 
    & \qquad
    +
    {O}
    \Bigg(
    L(\cU_0)\sqrt{\log(2MN/\delta)}
    \sum_{z=2}^{N}
    \ind 
    \{
    S_z \ge 
    2M^2 \log(NM/\delta)
    \}
    \sqrt{\frac{M}{S_z}}
    \Bigg)
    ,\label{eqn:comp-gap}
\end{align}
where $ L(\cU_0) = \frac{F(c\Delta_{u^*}^3)
    -
    F(\Delta_{u^*})}
    { \Delta_{u^*}
    - c\Delta_{u^*}^3 }$ is an instance-dependent factor, with only logarithmic dependence on $N$ and $\delta^{-1}$(through $F$), and where $c = 1/25$ is a global constant. 
\end{theorem}

\begin{proof}
The detailed proof can be found in Appendix~\ref{subsec:proof-comp-gap}.
\end{proof}

A few discussions are necessary to show the meaning of the result. First, if the number of users $M \gg N$, then no user is eliminated because each user will be queried so few times that no meaningful inference can be made. Since the goal is to achieve the accuracy of the best user, more inaccurate users only make the task more difficult. Therefore, it is necessary to impose assumptions on $M$ with respect to $N$.


This intuition can be made more precise. 
Suppose we loosely bound $S_t$ as $S_t \ge t \log(t/\delta)$, which is reasonable since for a very accurate user the algorithm will spend roughly no more than $O(\log(t/\delta))$ comparisons to insert one item.
This means the complexity can be bounded as (ignoring log factors)
\begin{align}
    \cC_{\mathrm{Alg}}(N,M)
    & = 
    O\big( N F(\Delta_{u^*})
    \big)
    \notag 
    +
    \tilde{O}
    \big(
    M^2
    \big (F(\bar{\Delta}_0)
    -
    F(\Delta_{u^*}) \big)
    \big)
    \notag \\ 
    & \qquad
    +
    \Tilde{O}
    \big(
    L(\cU_0)
    \big(
    \sqrt{M}
    (\sqrt{N} - M)
    \big)
    \big)
    . \label{eqn:comp-gap-simp}
\end{align}
If $M=\Omega(\sqrt{N})$, then this is not ideal because our algorithm won't eliminate any user until $\Omega(N)$  items are inserted with accuracy $\bar{\Delta}_0$, which already leads to a gap linear in $N$ compared with the best complexity $\cC_{u^*}$. In this case, our algorithm roughly makes the same amount of queries as $\cC_{\mathrm{ave}}$. 

In order to avoid the bad case,  it is necessary to assume $M = o(\sqrt{N})$ so that the last two terms become negligible (notice that $L(\cU_0)$ is an instance-dependent constant). 
Now we restate Theorem~\ref{theorem:complexity-gap} with the additional assumption, and compare it with the baselines. 
\begin{proposition}\label{prop:complexity-gap}
Suppose we have $M$ users and $N$ items to rank exactly, with $M = o(\sqrt{N})$. We have the following complexity along with~\eqref{eqn:comp-best-user} and~\eqref{eqn:comp-aver-user}:
\begin{align*}
    \cC_{u^*}(N,M)
    & = 
    O(
    N F(\Delta_{u^*})
    )
    \\
    \cC_{\mathrm{ave}}(N,M)
    & = 
    O(
    N F(\bar{\Delta}_0)
    )
    \\
    \cC_{\mathrm{Alg}}(N,M)
    & = 
    O(
    N F(\Delta_{u^*})
    )
    +
    o
    (N
    \big (F(\bar{\Delta}_0)
    -
    F(\Delta_{u^*}) \big)
    )
    +
    o
    \big(
    N
    \big).
\end{align*}
The last two terms of $\cC_{\mathrm{Alg}}(N,M)$ are negligible when compared with the first term. Therefore, our algorithm can perform comparably efficiently as if the best user is known while enjoying an advantage over the naive algorithm with sample complexity $\cC_{\mathrm{ave}}(N,M)$.
\end{proposition}
\begin{remark}
Note that if we set $\cU_0 = \{u^* \}$ for our algorithm, our algorithm will achieve exactly the same complexity as~\eqref{eqn:comp-best-user} indicates. Similarly, if we construct a new user $\bar{u}$ where $\Delta_{\bar{u}} = \bar{\Delta}_0$ and set $\cU_0 = \{ \bar{u}\}$, our algorithm will recover exactly~\eqref{eqn:comp-aver-user}.
By this argument and the fact that Big-$O$ notations hide no $M$, the first term in each equation actually has the same absolute constant factor. 
Therefore, our algorithm is indeed comparable with the best user.
\end{remark}

\begin{remark}
Notice that $F(x) \rightarrow +\infty$ when $x \rightarrow 0$. This means $\cC_{\mathrm{ave}}$ is very sensitive to the initial average accuracy margin $\bar{\Delta}_0$. In the case where there is only one best user $u^*$ and all other users have a near-zero margin  $\Delta_u \rightarrow 0$, $\cC_{\mathrm{ave}}$ can be very large compared with $\cC_{u^*}$.
\end{remark}
\begin{remark}
In the experiments, we notice that even with $N=10$ and $M=9$, after inserting the first item, each user has already been queried for enough times so that $S_2 \ge 2M^2 \log(NM/ \delta)$, which makes the second term in~\eqref{eqn:comp-gap} vanish.
\end{remark}

\section{Experiments}\label{sec:experiment}

In this section, we study the empirical performance of the proposed algorithm through a synthetic experiment. The following four algorithms are compared.
\begin{itemize}[leftmargin=*]
    \item \textbf{Non-adaptive user sampling} \citep{ren2019sample}: The original algorithm does not distinguish between users. For each comparison, we query a user selected uniformly at random.
    \item \textbf{Adaptive user sampling}: The proposed method.
    \item \textbf{Two stage ranking}: First, an arbitrary pair $(i,j)$ is chosen, and each user is queried for enough times so that the order can be determined with high probability. Second, based on the predictions of each user, a near-best user is identified. The ranking task is then performed by only querying the selected user. More details are presented in Section \ref{sec:two-stage} of the supplementary material.
    \item \textbf{Oracle}: Query only the best user as if it is known.
\end{itemize}

In our experiment, we use a similar setup as that of \citet{jin2020rank}. In particular, we consider a set of users $[M]$, whose accuracies are set by $p_u(i, j) = (1 + \exp(\gamma_u (s_j - s_i)))^{-1}$, for $u\in[M]$ and any items $i,j\in[N]$, where parameter $\gamma_u$ is used as an scaling factor of the user accuracy and $s_i, s_j$ are the utility scores of the corresponding items in the BTL model. The larger the scaling factor $\gamma_u$ is, the more accurate the user $u$ is. We set $s_i-s_j = 3$ if $i \prec j$ and $s_i - s_j = -3$ otherwise. Note that here we assume that the accuracy of user $u$ is the same for all pair of items $(i,j)$ as long as $i\prec j$. We assume that there are two distinct groups of users: the high-accuracy group in which the users have the same accuracy $\gamma_u = \gamma_B \in [0.5, 1.0, 2.5]$ in three different settings respectively; and the low-accuracy group in which the users have the same accuracy $\gamma_u = \gamma_A = 0.5$ in all settings. 
This set of $\gamma_u, s_i, s_j$ is chosen so that $p_u(i,j)$ for accurate users ranges from $0.55$ to $0.99$ and inaccurate users have a value close to $0.55$.

The number of items to be ranked ranges from 10 to 100. Each setting is repeated 100 times with randomly generated data. To showcase the effectiveness of active user selection, we tested a relatively adverse situation where only $3$ out of $M = 9$ users are highly accurate. The $p_u(i,j)$ is generated in a parameterized way controlled by $\gamma_u, s_i, s_j$. In each run, given the $\gamma_u$ for each user,

The average sample complexity over 100 runs and standard deviation are plotted in Figure \ref{fig:exp1}. Note that the standard deviation is visually indistinguishable compared to the average. Actual values can be found in the supplementary material. 
In most cases, the proposed method achieve nearly identical performance as the oracle one does, with only a small overhead. The gap is indistinguishable in the graph and the actual values are listed in the supplementary material. For two-stage algorithm, we observe a constant overhead regardless the accuracy of users. It may out perform the naive one if there exist enough highly accurate users such as in Figure \ref{fig:g25}. However, the situation is less favorable for the two-stage algorithm when the cost of finding the best user overwhelms the savings of queries due to increased accuracy as shown in Figure \ref{fig:g10}. It may even have an adverse effect in face of a user base with no expert users as shown in Figure \ref{fig:g05}. Nonetheless, our proposed method can adapt to each case and deliver near optimal performance. In our experiments every algorithm is able to recover the exact rank with respect to the ground truth, which is reasonable since the IIR algorithm is designed to output an exact rank. And due to the union bounds used to guarantee a high probability correct output, the algorithm tends to request more than enough queries so that we did not see the case when it output a non-exact rank.


\begin{figure}[H]
\centering
\subfigure[$\gamma_A=0.5$, $\gamma_B=2.5$]{\includegraphics[width=0.32\textwidth]{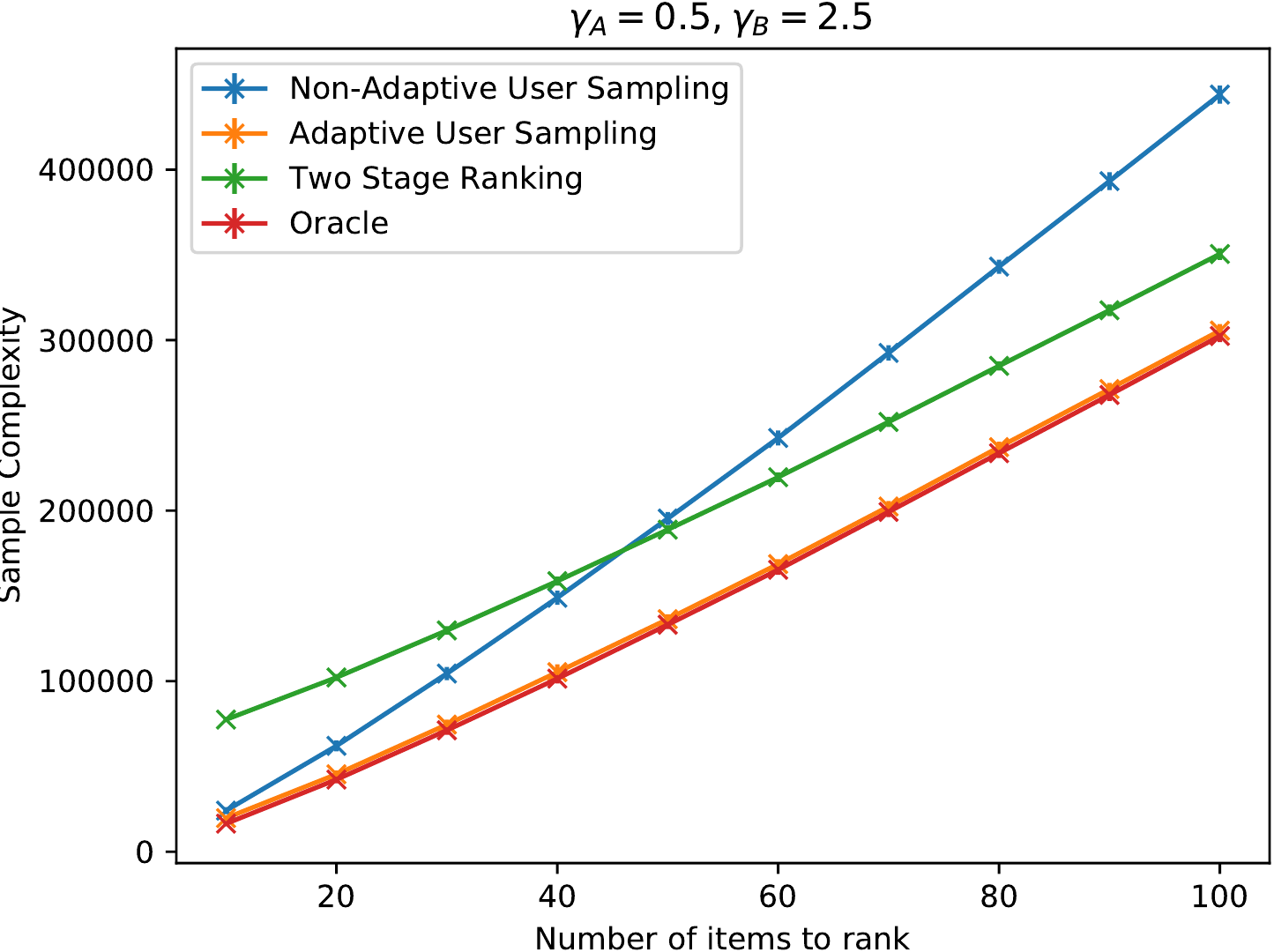}\label{fig:g25}}
\subfigure[$\gamma_A=0.5$, $\gamma_B=1.0$]{\includegraphics[width=0.32\textwidth]{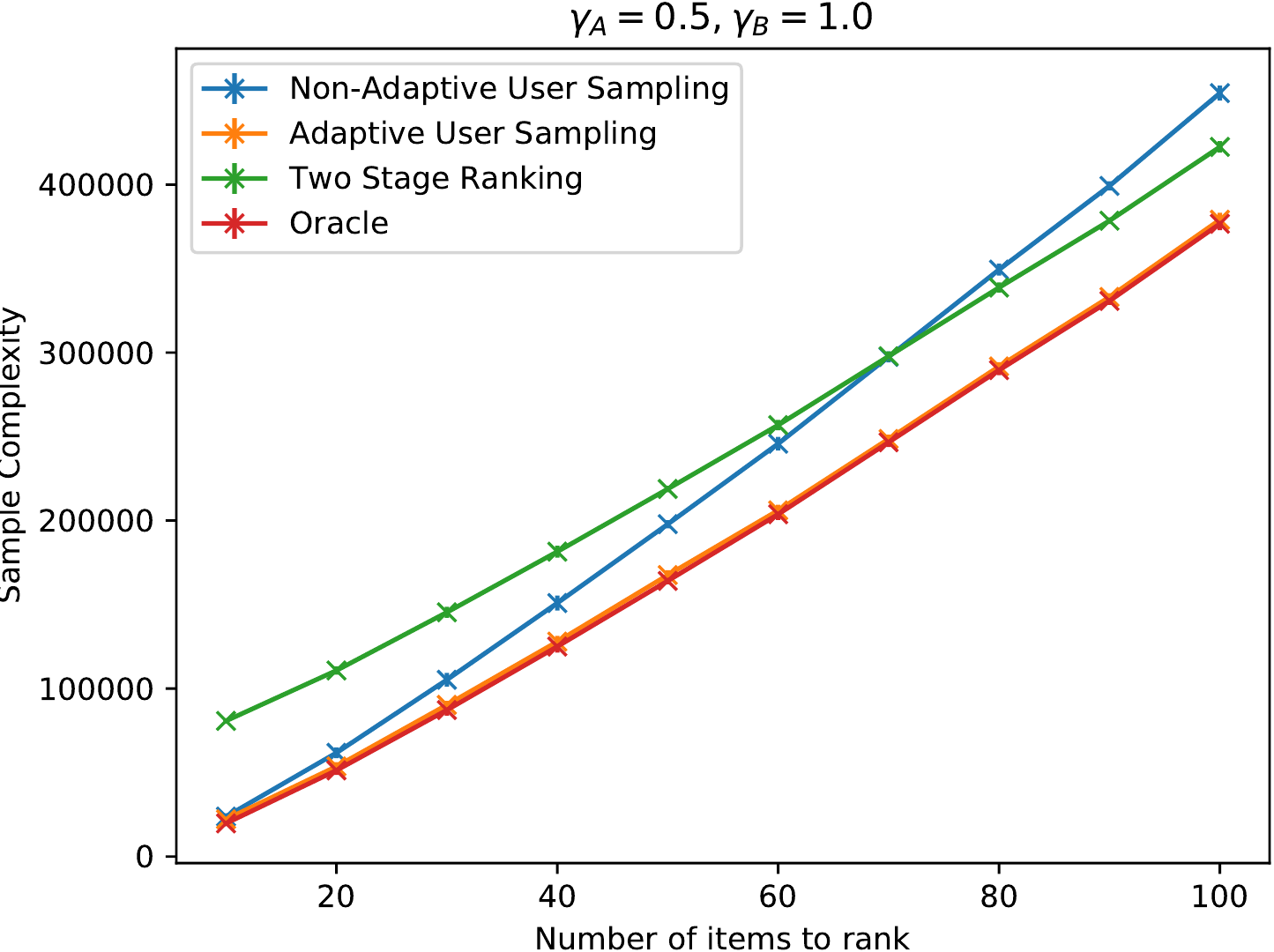} \label{fig:g10}}
\subfigure[$\gamma_A=0.5$, $\gamma_B=0.5$]{\includegraphics[width=0.32\textwidth]{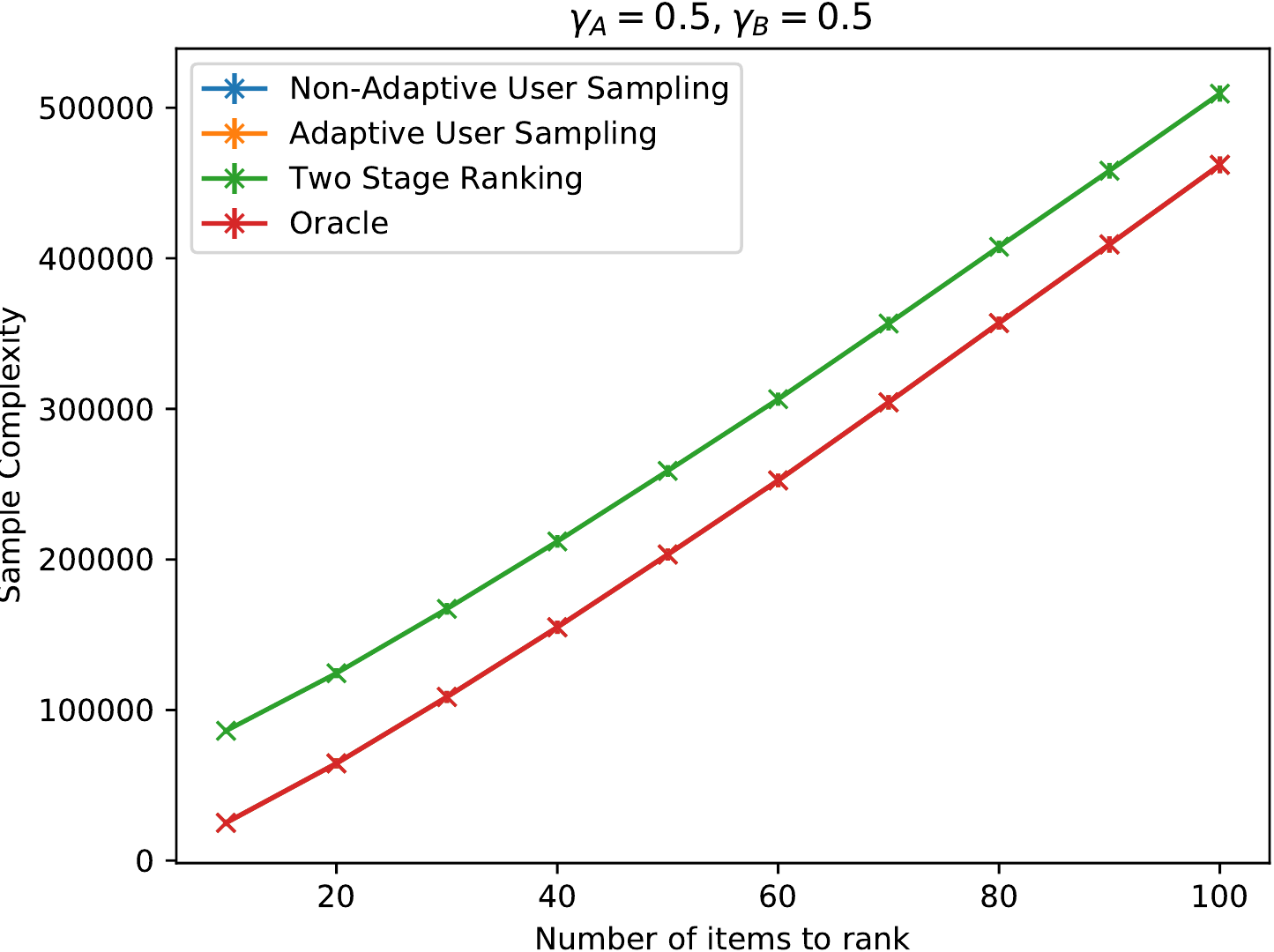} \label{fig:g05}}

\caption{Sample complexities v.s. number of items for all algorithms. (a) (b) and (c) are different heterogeneous user settings where the accuracy of two group of users differs.
\label{fig:exp1}}
\end{figure}


\section{Conclusions}\label{sec:conclusion}
In this paper, we study the heterogeneous rank aggregation problem, where noisy pairwise comparisons are provided by a group of users with different accuracy levels. We propose a new ranking algorithm based on the idea of arm elimination from multi-armed bandits. The algorithm can identify the best user and utilize this information to efficiently perform the ranking. 
Under the Bernoulli setting, we provide theoretical guarantees that our algorithm is comparable with the oracle algorithm that uses the best user to perform single-user ranking, and the gap between sample complexities of these two methods is only sublinear in the number of items.
We conduct thorough experiments and show that the proposed algorithm can perform as good as the oracle algorithm and is significantly more sample efficient than all baseline algorithms. One immediate and interesting future direction may be to extend our adaptive sampling algorithm to more complicated models such as the heterogeneous Bradley-Terry-Luce model and the heterogeneous  Thurstone
Case V model   \citep{jin2020rank}.

\newpage

\appendix
\section{More Details About the Proposed Algorithm}\label{sec:omitted_subrountine}
We borrow the definition of PIT  from  \citet{feige1994computing,ren2019sample}, based on which we can insert items to  a ranked list. Specifically, given a list of ranked items $S$ the PIT can be constructed using the following Algorithm \ref{alg:pit}. 

\begin{algorithm}
\caption{Build PIT}\label{alg:pit}
\textbf{Input parameters:} $S$ \\ 
\textbf{Data structure:} \textsc{Node} = $\{left, mind, right, lchild, rchild, parent \}$, $left, mid, right$ holds index values, $lchild, rchild, parent$ points to any other \textsc{Node}. \\
\textbf{Initialize:} $N = |S|$
\begin{algorithmic}[1]
    \STATE $X = \textsc{CreateEmptyNode}$ {\color{gray}{returns an empty Node with above mentioned data structure}}
    \STATE $X.\text{left} = -1$
    \STATE $X.\text{right} = |S|$
    \STATE $X.\text{mid} = \lfloor{(X.left + X.right)/2 \rfloor}$
    \STATE queue = [X]
    \WHILE {queue.\textsc{NotEmpty}}
        \STATE X = queue.\textsc{PopFront}
        \STATE $X.\text{mid} = \lfloor{(X.left + X.right)/2 \rfloor}$ 
        \IF {X.right - X.left $>$ 1}
            \STATE lnode = \textsc{CreateEmptyNode}
            \STATE lnode.left = X.left
            \STATE lnode.right = mid
            \STATE X.lchild = lnode
            \STATE rnode = \textsc{CreateEmptyNode}
            \STATE queue.append(lnode)
            \STATE rnode.left = X.mid
            \STATE rnode.right = X.right
            \STATE X.rchild = rnode
            \STATE queue.append(rnode)
        \ENDIF
    \ENDWHILE
    \STATE replace $-1$ with $-\infty$, $|S|$ with $\infty$ in each \textsc{Node}.left and \textsc{Node}.right.
\end{algorithmic}
\end{algorithm}

For the completeness of our paper, we also present the subroutines \textsc{Iterative-Attempting-Insertion} (IAI) and \textsc{Attempting-Insertion} (ATI) in this section, which are omitted in Section \ref{sec:alg_elim} due to space limit. In particular, IAI is displayed in Algorithm \ref{alg:iai} and ATI is displayed in Algorithm \ref{alg:ati}. Both algorithms are proposed by \citet{ren2019sample} for adaptive sampling in the single user setting. 

\begin{algorithm}[ht!]
\caption{Subroutine: \textsc{Iterative Attempt To Insert}(IAI)}\label{alg:iai}
\textbf{Input parameters:} $(i,S,\delta)$ \\ 
		\textbf{Initialize:} For all $\tau\in\mathbb{Z}^+$, set $\epsilon_\tau={2^{-(\tau+1)}}$ and $\delta_\tau=\frac{6\delta}{\pi^2 \tau^2}$; $t\gets 0$; $Flag\gets$ \textit{unsure};
		\begin{algorithmic}[1]
			\REPEAT 
			    \STATE $t\gets t+1$;
			    \STATE $Flag\gets$ATI$(i,S,\epsilon_z,\delta_z)$;
			\UNTIL{$Flag=$}\textit{ inserted} 
		\end{algorithmic}
\end{algorithm}

\begin{algorithm}[ht!]
\caption{Subroutine: \textsc{Attempt To Insert}(ATI).}\label{alg:ati}
    \textbf{Input parameters:} $(i,S,\epsilon, \delta)$ \\
	\textbf{Initialize:}
	Let $z$ be a PIT constructed from $S$, $h\gets\lceil 1 + \log_2(1+|S|)\rceil$, the depth of $z$ \\
	For all leaf nodes $u$ of $z$, initialize $c_u\gets 0$; Set $t^{\max}\gets \lceil\max\{4h, \frac{512}{25}\log\frac{2}{\delta}\}\rceil$ and $q\gets \frac{15}{16}$
	\begin{algorithmic}[1]
		\STATE $X\gets$ the root node of $z$;
		\FOR{$t \gets$ $1$ to $t^{\max}$}
		\IF{$X$ is the root node}
		    \IF{ATC$(i,X\mbox{.mid},\epsilon,1-q)$ = $i$} 
		        \STATE $X\gets X\mbox{.rchild}$ 
		    \ELSE
		        \STATE $X\gets X\mbox{.lchild}$
		    \ENDIF
		\ELSIF{$X$ is a leaf node}{
			\IF{ATC$(i,X\mbox{.left},\epsilon,1-\sqrt{q})=i$ $\land$ ATC$(i,X\mbox{.right},\epsilon,1-\sqrt{q})= X\mbox{.right}$}
			    \STATE $c_X\gets c_X+1$
			    \IF{$c_X> b^t := \frac{1}{2}t + \sqrt{\frac{t}{2}\log\frac{\pi^2 t^2}{3\delta}} + 1$}
			        \STATE Insert $i$ into the corresponding interval of $X$ and \RETURN {\textit{inserted}}
			    \ENDIF
    		\ELSIF{$c_X>0$} 
    		    \STATE $c_X\gets c_X-1$
    		\ELSE
    			\STATE $X\gets X.\mbox{parent}$
    		\ENDIF
		}
		\ELSE
		    \IF{ATC$(i,X\mbox{.left},\epsilon,1-\sqrt[3]{q})=X\mbox{.left}$ $\lor$ ATC$(i,X\mbox{.right},\epsilon,1-\sqrt[3]{q})=i$}
		        \STATE $X\gets X.\mbox{parent}$
		    \ELSIF{ATC$(i,X\mbox{.mid},\epsilon,1-\sqrt[3]{q})=i$}
		        \STATE $X\gets X\mbox{.rchild}$
		    \ELSE
		        \STATE $X\gets X\mbox{.lchild}$
		    \ENDIF
		\ENDIF
		\ENDFOR
		\IF{there is a leaf node $u$ with $c_u\geq 1+\frac{5}{16}t^{\max}$}
		    \STATE Insert $i$ into the corresponding interval of $u$ and \RETURN {\textit{inserted}}
		\ELSE
		    \RETURN \textit{unsure}
		\ENDIF
	\end{algorithmic}
\end{algorithm}

\section{A Two Stage Algorithm}\label{sec:two-stage}

In this section, we present a simple algorithm that performs user selection and item ranking in two separate stages. We then compare it with the baseline algorithms as well as the active sampling algorithm we proposed in Section \ref{sec:alg_elim}. 

\subsection{Algorithm Outline}
The two-stage algorithm first performs user-selection and then item-ranking. In the user-selection stage, we search for an $\alpha$-optimal user for some small $\alpha$. Next, we discard all other users and rank items by only taking queries from the selected user. Ranking with a single user is done by the Iterative-Insertion-Ranking (IIR) algorithm. Recall from Proposition~\ref{prop:comIIR} that IIR takes a set $\cN$ of $N$ items, a real number $\delta$ and a user $u$ (accuracy $\Delta_u$) as inputs, and outputs an exact ranking of $\cN$ with probability at least $1-\delta$ by taking 
\begin{align}
\Theta
    \big(
    N \Delta_{u}^{-2}
    \big(
    \log \log \Delta_{u}^{-1}
    +
    \log(N/\delta)
    \big)
    \big),    \label{eq:iircomplexity}
\end{align}
queries from user $u$. 

The user-selection stage, Naive-User-Selection (NUS), is presented in Algorithm~\ref{alg:naiveuserselection}. In NUS, we first take an arbitrary pair of items and run the IIR algorithm on them to determine the order. Note that at this point, users have not been distinguished yet. So we take each query from a randomly chosen user, and this is equivalent of querying the user $\bar u$ with accuracy $\bar \Delta$. After determining the order of the chosen item pair, the problem of finding an $\alpha$-optimal user is reduced to pure exploration of an $\alpha$-optimal arm in the multi-armed bandit problem, for which we adopt the Median-Elimination (ME) algorithm from~\cite{even2002pac}. ME takes a set $\cU$ of users, real numbers $\alpha,\delta$, and two ranked items as inputs, and outputs an $\alpha$-optimal user in $\cU$ with probability at least $1-\delta$ using 
\begin{align}
    \Theta\parenv{\frac{|\cU|}{\alpha^2}\log\frac{1}{\delta}} \label{eq:MEcomplexity}
\end{align} 
comparisons~\citep{even2002pac,mannor2004sample}.

\begin{algorithm}
\caption{Subroutine: \textsc{Naive-User-Selection}$(\cU,\alpha,\delta_i,\delta_m, i,j)$}\label{alg:naiveuserselection}
\begin{algorithmic}
    \STATE \textbf{input:} set of users $\cU$, desired near-optimal level $\alpha$, confidence level $\delta_i$ of initial ranking, confidence level $\delta_m$ of user selection, two items $i,j\in\cN$.
    \STATE $[i',j'] \leftarrow$ Iterative-Insertion-Ranking$(\{i,j\},\delta_i,\bar u)$. 
    \STATE \textbf{output:} Median-Elimination$(\cU,\alpha,\delta_m, [i',j'])$
\end{algorithmic}
\end{algorithm}
\begin{theorem}\label{thm:NUS}
For any $\delta_i,\delta_m\in(0,\frac{1}{2}),\alpha\in (0,\Delta_{u^*})$, with probability at least $1-\delta_i-\delta_m$, subroutine Naive-User-Selection$(\cU,\alpha,\delta_i,\delta_m,i,j)$ outputs a global $\alpha$-optimal user after
\begin{align}
    \Theta\parenv{{\bar\Delta}^{-2}\parenv{\log\log{\bar\Delta}^{-1}+\log\frac{1}{\delta_i}} + \frac{M}{\alpha^2}\log\frac{1}{\delta_m}} \label{eq:NUScomplexity}
\end{align}
comparisons.
\end{theorem}
The confidence bound in Theorem~\ref{thm:NUS} follows from applying the union bound on the called subroutines IIR and ME. The comparison complexity guarantee is obtained by summing up the complexities of IIR and ME given by \eqref{eq:iircomplexity} and \eqref{eq:MEcomplexity}.

After finding an $\alpha$-optimal user, the total ranking can be obtained by directly applying the IIR algorithm, presented in Algorithm~\ref{alg:TSR}.
\begin{algorithm}
\caption{\textsc{Two-Stage-Ranking}$(\cN,\cU,\alpha,\delta_i,\delta_m,\delta_r)$}
\begin{algorithmic}\label{alg:TSR}
    \STATE \textbf{input:} set of items $\cN$, set of users $\cU$, desired near-optimal level $\alpha$, confidence level $\delta_i$ of the initial ranking, confidence level $\delta_m$ of user selection, confidence level $\delta_r$ of the final ranking.
    \STATE Let $i,j$ be two arbitrary items.
    \STATE $u^\alpha\leftarrow$ Naive-User-Selection$(\cU,\alpha,\delta_i,\delta_m,i,j)$.
    \STATE \textbf{output:} Iterative-Insertion-Ranking$(\cN,\delta_r,u^\alpha)$
\end{algorithmic}
\end{algorithm}

\begin{theorem}\label{thm:TSR}
For any $\delta_i,\delta_m,\delta_r\in(0,\frac{1}{2}),\alpha\in(0,\Delta_{u^*})$, with probability at least $1-\delta_i-\delta_r-\delta_m$, Two-Stage-Ranking$(\cN,\cU,\alpha,\delta_i,\delta_m,\delta_r)$ outputs the exact ranking of $\cN$, and consumes
\begin{align*}
 \Theta\parenv{\!{\bar\Delta}^{-2}\!\parenv{\log\log{\bar\Delta}^{-1}\!+\!\log \frac{1}{\delta_i}} \!+\! \frac{M}{\alpha^2}\log \frac{1}{\delta_m}\!+\!N\parenv{\Delta_{u^*}\!-\!\alpha}^{-2}\!\parenv{\log\log \parenv{\Delta_{u^*}\!-\!\alpha}^{-1} \!+\! \log\frac{N}{\delta_r}}}
\end{align*}
comparisons.
\end{theorem}
The confidence bound in Theorem~\ref{thm:TSR} again follows directly from applying the union bound on the called subroutines. The comparison complexity guarantee is obtained by summing up the complexities of NUS and IIR given in \eqref{eq:NUScomplexity} and \eqref{eq:iircomplexity}.

\subsection{Complexity Analysis}
In this subsection, we provide a more detailed discussion on the complexity of the two-stage algorithm described in Algorithm \ref{alg:TSR}. 
Recall that we define 
\begin{align*}
    F(x) = x^{-2}\parenv{\log\log x^{-1}+\log\parenv{N/\delta}}.
\end{align*}
Let $\cC_{\mathrm{tsr}}(\alpha)$ be the complexity of the two-stage algorithm given by Theorem~\ref{thm:TSR}. Since the complexities are only given in the form of their order of magnitudes, in the following analysis we are safe to assume $\delta_i=\delta_m=\delta_r=\delta/3$. It follows that with probability at least $1-\delta$, TSR outputs the exact ranking and consumes number of comparisons
\begin{align*} 
\cC_{\mathrm{tsr}}(\alpha)=
\Theta\parenv{{\bar\Delta}^{-2}\!\parenv{\log\log{\bar\Delta}^{-1}\!+\!\log \frac{1}{\delta}} + \frac{M}{\alpha^2}\log \frac{1}{\delta}+N F\parenv{\Delta_{u^*}-\alpha}}.
\end{align*}
Note that the factor of 3 on $\delta$ is absorbed into the $\Theta$ notation. 

With $\delta$ being a constant, the following propositions can be made.
\begin{proposition}\label{prop:tsrlargeM}
When $M=\omega(N\log N)$ or $\alpha=o\parenv{\sqrt{\frac{M}{N\log N}}}$,
\begin{align*}
    \cC_{\mathrm{tsr}}(\alpha)=\omega(N\log N) + \Theta\parenv{NF(\Delta_{u^*}-\alpha)}. 
\end{align*}
\end{proposition}
When the number of users $M$ is too large or we require $\alpha$ to be too small, the number of comparisons used in user selection becomes even larger than doing ranking naively, and is undesirable.
\begin{proposition}\label{prop:tsrgood}
If $M=O(N)$ and $\alpha=\omega\parenv{\sqrt{\frac{M}{N\log N}}}\cap o(1)$, then
\begin{align*}
    \cC_{\mathrm{tsr}}(\alpha) = \Theta\parenv{NF(\Delta_{u^*})} + o(N\log N) +O(1).
\end{align*}
\end{proposition}
By the preceding proposition, when $M=O(N)$ and $\alpha=\omega\parenv{\sqrt{\frac{M}{N\log N}}}\cap o(1)$,
the two-stage algorithm has complexity $\Theta\parenv{\cC_{\Delta_{u^*}}}$ plus lower order terms. Therefore, we do not need to spend too many comparisons in user selection while still achieving a ranking performance close to optimal.
\begin{proposition}
If $M=O(N)$ and $\alpha$ is a constant,
\begin{align*}
 \cC_{\mathrm{tsr}}(\alpha) = \Theta\parenv{NF(\Delta_{u^*}-\alpha)} + O(M) .
\end{align*}
\end{proposition}
In this case, the selected user has a constant accuracy gap $\alpha$ from the best user. However, by choosing $\alpha$ small enough, the complexity can get close to optimal (still with a linear gap).

\subsection{User Selection in a Subset}
As shown in Proposition~\ref{prop:tsrlargeM}, when $M$ is much larger than $N\log N$, even querying each user once costs time linear in $M$ which could be higher than the ranking complexity. Therefore, instead of selecting a global $\alpha$-optimal user, we devise a subroutine Subset-User-Selection (SUS) that randomly picks without replacement $L$ ($L\le M$) users and only search for an $\alpha$-optimal user among them (see Algorithm~\ref{alg:subsetuserselection}). We use $\cL$ to denote this $L$-subset of users. 
\begin{algorithm}[H]
\caption{Subroutine: \textsc{Subset-User-Selection}$(\cU,L,\alpha,\delta_i,\delta_m, i,j)$}\label{alg:subsetuserselection}
\begin{algorithmic}
    \STATE \textbf{input:} set of users $\cU$, user subset size $L$, desired near-optimal level $\alpha$, confidence level $\delta_i$ of initial ranking, confidence level $\delta_m$ of user selection, two items $i,j\in\cN$.
    \STATE $[i',j'] \leftarrow$ Iterative-Insertion-Ranking$(\{i,j\},\delta_i,\bar u)$. 
    \STATE Randomly choose a subset $\cL$ of $L$ users from $\cU$.
    \STATE \textbf{output:} Median-Elimination$(\cL,\alpha,\delta_m, [i',j'])$
\end{algorithmic}
\end{algorithm}
The main procedure of the two-stage algorithm is also modified, shown in Algorithm~\ref{alg:MTSR}. 
\begin{algorithm}[H]
\caption{\textsc{Modified-Two-Stage-Ranking}$(\cN,\cU,L,\alpha,\delta_i,\delta_m,\delta_r)$}
\begin{algorithmic}\label{alg:MTSR}
    \STATE \textbf{input:} set of items $\cN$, set of users $\cU$, user subset size $L$, desired near-optimal level $\alpha$, confidence level $\delta_i$ of initial ranking, confidence level $\delta_m$ of user selection, confidence level $\delta_r$ of final ranking.
    \STATE Let $i,j$ be two arbitrary items.
    \STATE $u^\alpha\leftarrow$ Subset-User-Selection$(\cU,L,\alpha,\delta_i,\delta_m,i,j)$.
    \STATE \textbf{output:} Iterative-Insertion-Ranking$(\cN,\delta_r,u^\alpha)$
\end{algorithmic}
\end{algorithm}

Generally, no guarantee can be made on how close is a subset $\alpha$-optimal user to the global optimal user. So analysis on the two-stage algorithm will be done under the assumption that the $M$ user accuracies are iid samples drawn from a probability distribution $F(x)$ over the interval $(0,\frac{1}{2}]$ ($F(x)$ is independent of any other quantities). Let $b=\inf_{x}\{x:F(x)=1\}$. In the following, we use the cdf $F(x)$ to represent this distribution.

Since $\Delta_1,\Delta_2,\ldots,\Delta_M$ are iid samples from $F(x)$ and $\cL$ is drawn randomly, we assume WOLOG that $\cL$ contains the first $L$ users, i.e., $\cL=\{1,2,\ldots,L\}$. Let $\Delta^\circ = \max_{u\in\cL} \Delta_u$. Recall that $\Delta_{u^*} = \max_{u\in\cU}\Delta_u$. We first show in the following lemma that $\Delta_{u^*}-\Delta^\circ$ is independent of $M$. 
\begin{lemma}\label{lem:alph-opt}
For any $\delta'\in(0,\frac{1}{2}),\alpha\in(0,b)$, if $L\ge \log(\delta')/\log\parenv{F(b-\alpha)}$, then with probability at least $1-\delta'$, 
\begin{align*}
    \Delta^\circ \ge \Delta_{u^*} -\alpha.
\end{align*}
\end{lemma}
\begin{proof}
Note that the claim becomes trivial when $M\le \frac{\log\delta'}{\log\parenv{F(b-\alpha)}} $. We consider the case when $\frac{\log\delta'}{\log\parenv{F(b-\alpha)}}\le L\le M$. 

Since $\Delta_1,\Delta_2,\ldots,\Delta_L$ are iid samples from $F(x)$, with probability $\parenv{F\parenv{b-\alpha}}^L$,
\begin{align*}
    \Delta_i \le b-\alpha \text{ for all } 1\le i\le L.
\end{align*}
Hence, $\parenv{F\parenv{b-\alpha}}^L\le \delta'$ gives
\begin{align*}
    \Delta^\circ=\max_{1\le i\le L}\Delta_i \ge b-\alpha \ge \Delta_{u^*}-\alpha
\end{align*}
with probability at least $1-\delta'$, where the last inequality follows from $\Delta_{u^*}\le b$ with probability 1.
\end{proof}
The preceding lemma states that when user accuracies follow a fixed distribution, at least one of the $L$ users we select randomly will be close to the global best user as long as $L$ is large enough (but still independent of $M$). Thus, even when the number of users $M$ is huge, we do not need to collect information from every one of them. A randomly chosen subset is able to accurately reflect the characteristics of the larger group.

Next, we compute the number of comparisons needed for user selection. Our goal is to show that the complexity of user selection becomes negligible compared with item ranking. In the following analysis, for simplification, we assign the confidence levels $\delta_i,\delta_m,\delta_r$ in Two-Stage-Ranking as well as the confidence level $\delta'$ for the existence of an $\alpha$-optimal user equal values. Specifically, we let $\delta'=\delta_i=\delta_m=\delta_r = \frac{\delta}{4}$ for some $\delta\in(0,1)$.
\begin{theorem}\label{thm:sus}
For any $\delta\in(0,\frac{1}{2}),\alpha\in (0,b),L =\min\parenv{\ceil{\frac{\log(\delta/4)}{\log\parenv{F(b-\alpha/2)}}}, M}$, with probability at least $1-\frac{3\delta}{4}$, subroutine Subset-User-Selection$(\cU,L,\frac{\alpha}{2},\frac{\delta}{4},\frac{\delta}{4},i,j)$ outputs a global $\alpha$-optimal user after
\begin{align*}
    \Theta\parenv{{\bar\Delta}^{-2}\parenv{\log\log{\bar\Delta}^{-1}+\log\frac{4}{\delta}} + \frac{4L}{\alpha^2}\log\frac{4}{\delta}}
\end{align*}
comparisons.
\end{theorem}
\begin{proof}
By Lemma~\ref{lem:alph-opt}, letting $L = \min\parenv{\ceil{\frac{\log(\delta/4)}{\log\parenv{F(b-\alpha/2)}}}, M}$ gives $\Delta^\circ\ge\Delta_{u^*}-\frac{\alpha}{2}$ with probability at least $1-\frac{\delta}{4}$. Moreover, IIR finds the correct order of items $i,j$ with probability at least $1-\frac{\delta}{4}$ and given that Median-Elimination outputs an $\frac{\alpha}{2}$-optimal user in the $L$-subset with probability at least $1-\frac{\delta}{4}$. Therefore, by the union bound, with probability $1-\frac{3\delta}{4}$, the $\frac{\alpha}{2}$-optimal user found is a global $\alpha$-optimal user. 

The complexity is a sum of two terms: the complexity of IIR ranking two items and the complexity of Median-Elimination outputting an $\alpha/2$-optimal user among $L$ users.
\end{proof}

\begin{theorem}\label{thm:MTSR}
For any $\delta\in(0,\frac{1}{2}),\alpha\in(0,b),L=\min\parenv{\ceil{\frac{\log (\delta/4)}{\log\parenv{F(b-\alpha/2)}}},M}$, with probability at least $1-\delta$, Modified-Two-Stage-Ranking$(\cN,\cU,L,\alpha,\frac{\delta}{4},\frac{\delta}{4},\frac{\delta}{4})$ outputs the exact ranking of $\cN$, and consumes
\begin{align*}
 \cC_{\mathrm{mtsr}}(\alpha)=\Theta\parenv{\!{\bar\Delta}^{-2}\!\parenv{\log\log{\bar\Delta}^{-1}\!+\!\log \frac{4}{\delta}} \!+\! \frac{4L}{\alpha^2}\log \frac{4}{\delta}\!+\!NF\parenv{\Delta_{u^*}\!-\!\alpha}}
\end{align*}
comparisons.
\end{theorem}
\begin{proof}
Modified-Two-Stage-Ranking being able to output the exact ranking of $\cN$ is guaranteed by the algorithm IIR. 

It remains to compute the complexity. By Theorem~\ref{thm:sus}, with probability at least $1-\frac{3}{4}\delta$, Subset-User-Selection outputs a global $\alpha$-optimal user. With a global $\alpha$-optimal user, IIR outputs the exact ranking of $\cN$ after
\begin{align*}
    \Theta\parenv{NF\parenv{\Delta_{u^*}-\alpha}},
\end{align*}
comparisons with probability at least $1-\frac{\delta}{4}$. Therefore, the desired complexity follows from applying the union bound and summing up the complexities of SUS and IIR.
\end{proof}

By noting that for $M$ sufficiently large, $\bar\Delta$ equals the mean of $F(x)$ with probability 1 and thus ${\bar\Delta}^{-2}\parenv{\log\log{\bar\Delta}^{-1}+\log \frac{4}{\delta}}=O(1)$, we have the following proposition.
\begin{proposition}
If $\alpha=\Omega(N^{-\frac{1}{2}})\cap o(1)$, then 
\begin{align*}
    \cC_{\mathrm{mtsr}}(\alpha) = \Theta\parenv{NF(\Delta_{u^*})} + O(N).
\end{align*}
Comparing the preceding proposition with Proposition~\ref{prop:tsrgood}, we can see that by Subset-User-Selection, the two-stage algorithm can perform efficiently even with a large number of users. 

\end{proposition}

\section{Additional Experiments}
In this section, we provide additional numerical experiments to demonstrate the advantage of our method. 
First we extend the accuracy of users to be generated in Section \ref{sec:experiment} to a wider range of parameters: $\gamma_A \in [0.25, 0.5, 1.0], \gamma_B \in [0.5, 1.0, 2.5]$. We also tested the performance of the algorithm when there are larger amount of users as $M=[9, 18, 36]$ and the result with different $\gamma_u$ configurations are shown in Fig. \ref{fig:exp-m9}, \ref{fig:exp-m18}, \ref{fig:exp-m36}. In each case, a portion of $\frac{1}{3}$ of the users have $\gamma_u = \gamma_B$, and the rest have $\gamma_u = \gamma_A$. Though the original `Medium Elimination' order optimal, its constant factor penalty is too large to be practical. We 
turn to use the successive elimination algorithm  \citep{even2002pac} as we did in Algorithm \ref{alg:elim-user} with $\epsilon = 0.15$ to identify the best user in the given result.

When comparing the same $\gamma_u$ setting with different users such as in Fig \ref{fig:m9gb0.25gg2.5}, \ref{fig:m18gb0.25gg2.5}, \ref{fig:m36gb0.25gg2.5}. The adaptive algorithm has similar performance with the two-stage one but without the overhead when there are smaller amout of items. It also shows when $M$ is increasing, the advantage of adaptive sampling is diminishing compared to the non-adaptive one due to the fact that the queries are spread over more users thus it takes longer to find the better set of more accurate users.


            \begin{figure}[ht!]
            \centering
            \subfigure[$\gamma_A=0.25$, $\gamma_B=0.5$]{\includegraphics[width=0.32\textwidth]{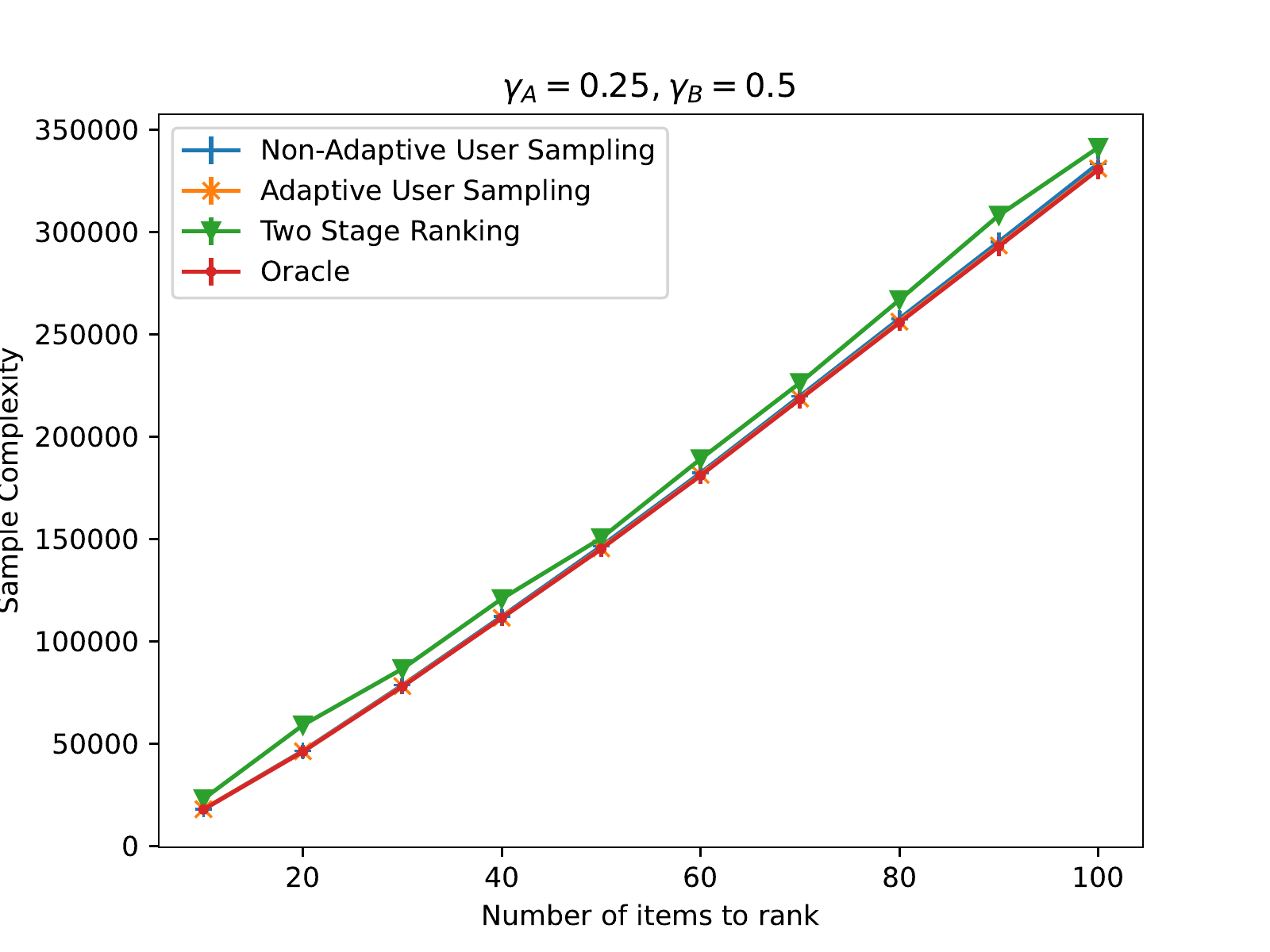} \label{fig:m9gb0.25gg0.5}}
\subfigure[$\gamma_A=0.25$, $\gamma_B=1.0$]{\includegraphics[width=0.32\textwidth]{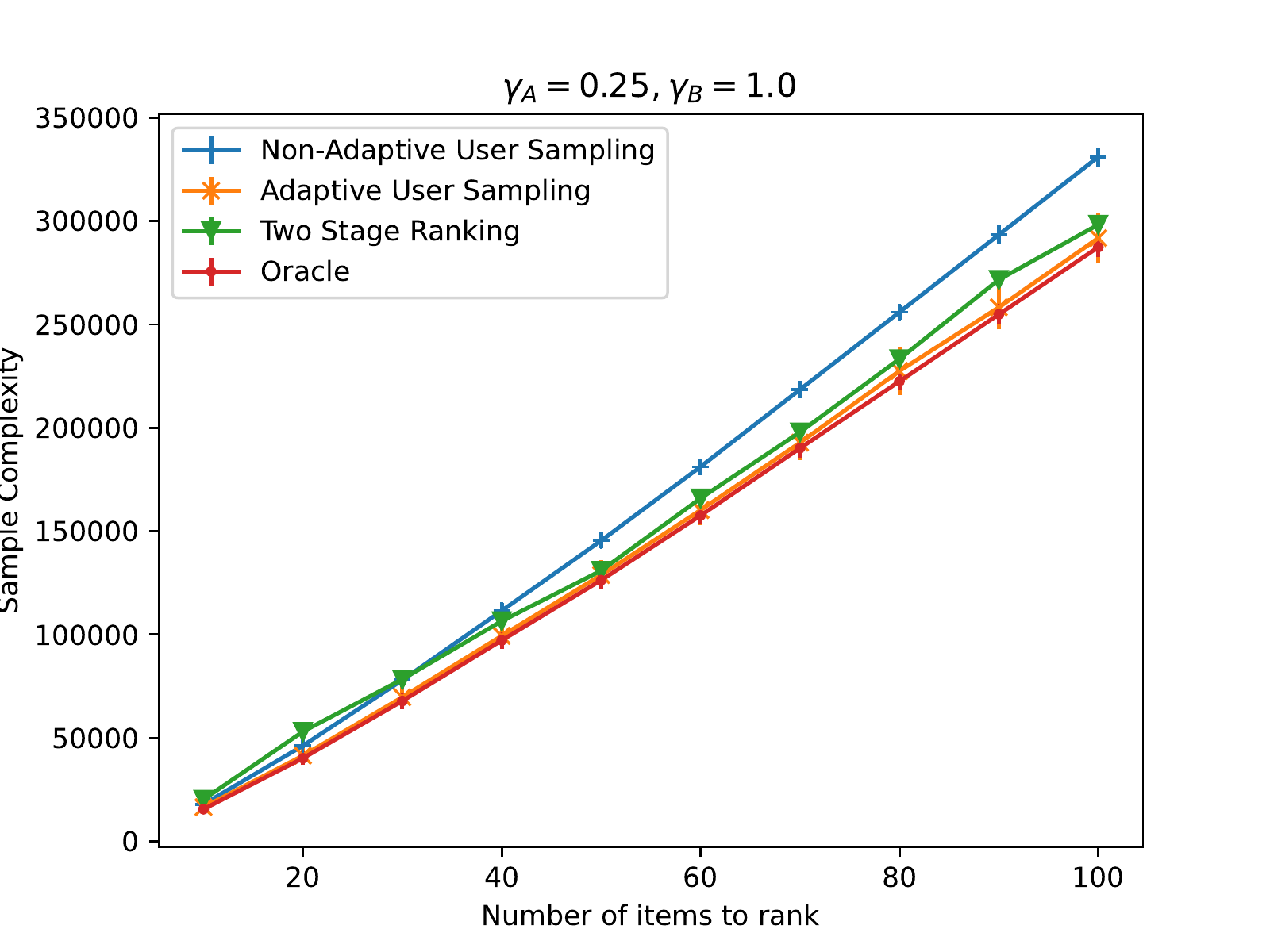} \label{fig:m9gb0.25gg1.0}}
\subfigure[$\gamma_A=0.25$, $\gamma_B=2.5$]{\includegraphics[width=0.32\textwidth]{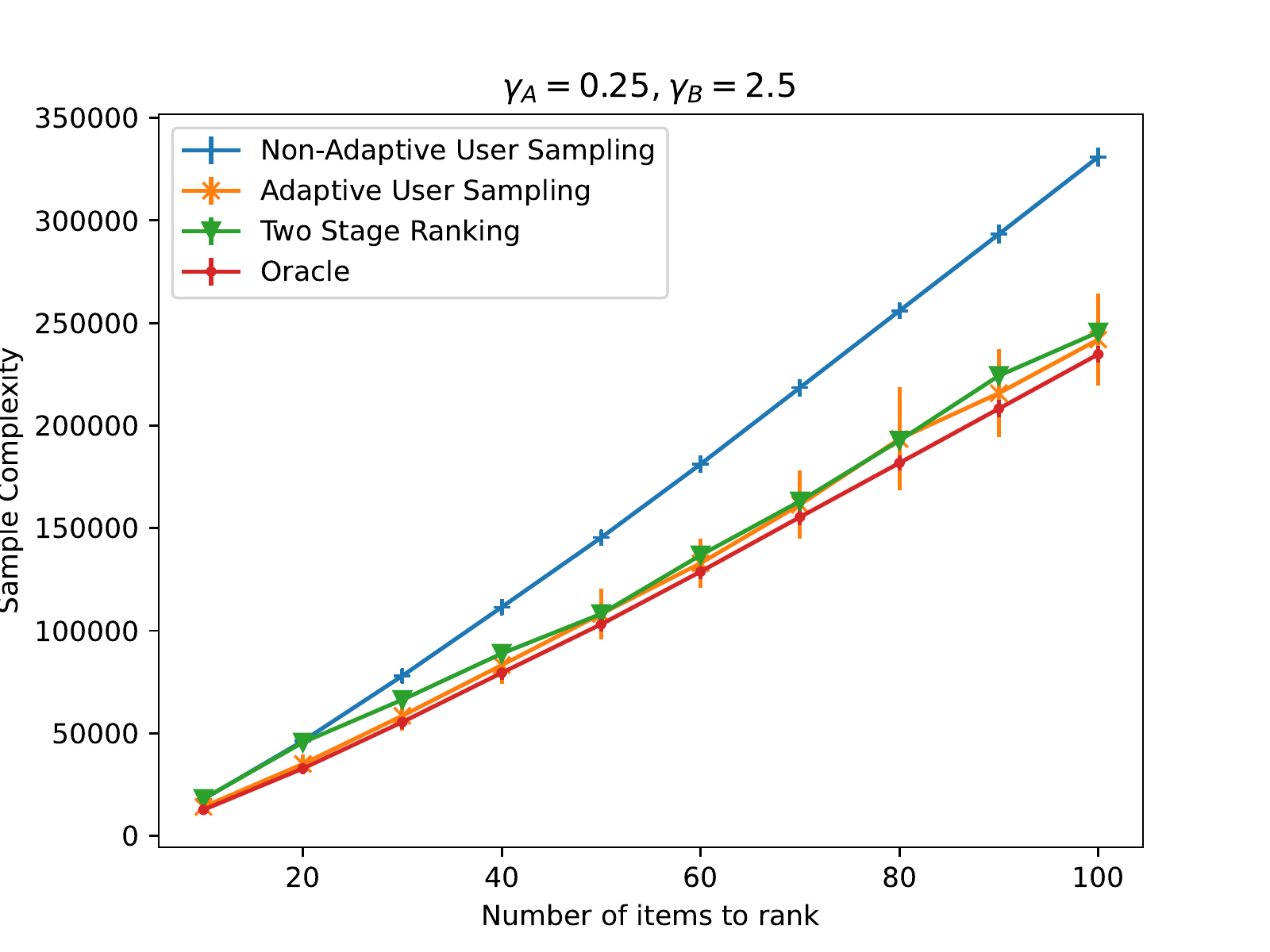} \label{fig:m9gb0.25gg2.5}}
\subfigure[$\gamma_A=0.5$, $\gamma_B=0.5$]{\includegraphics[width=0.32\textwidth]{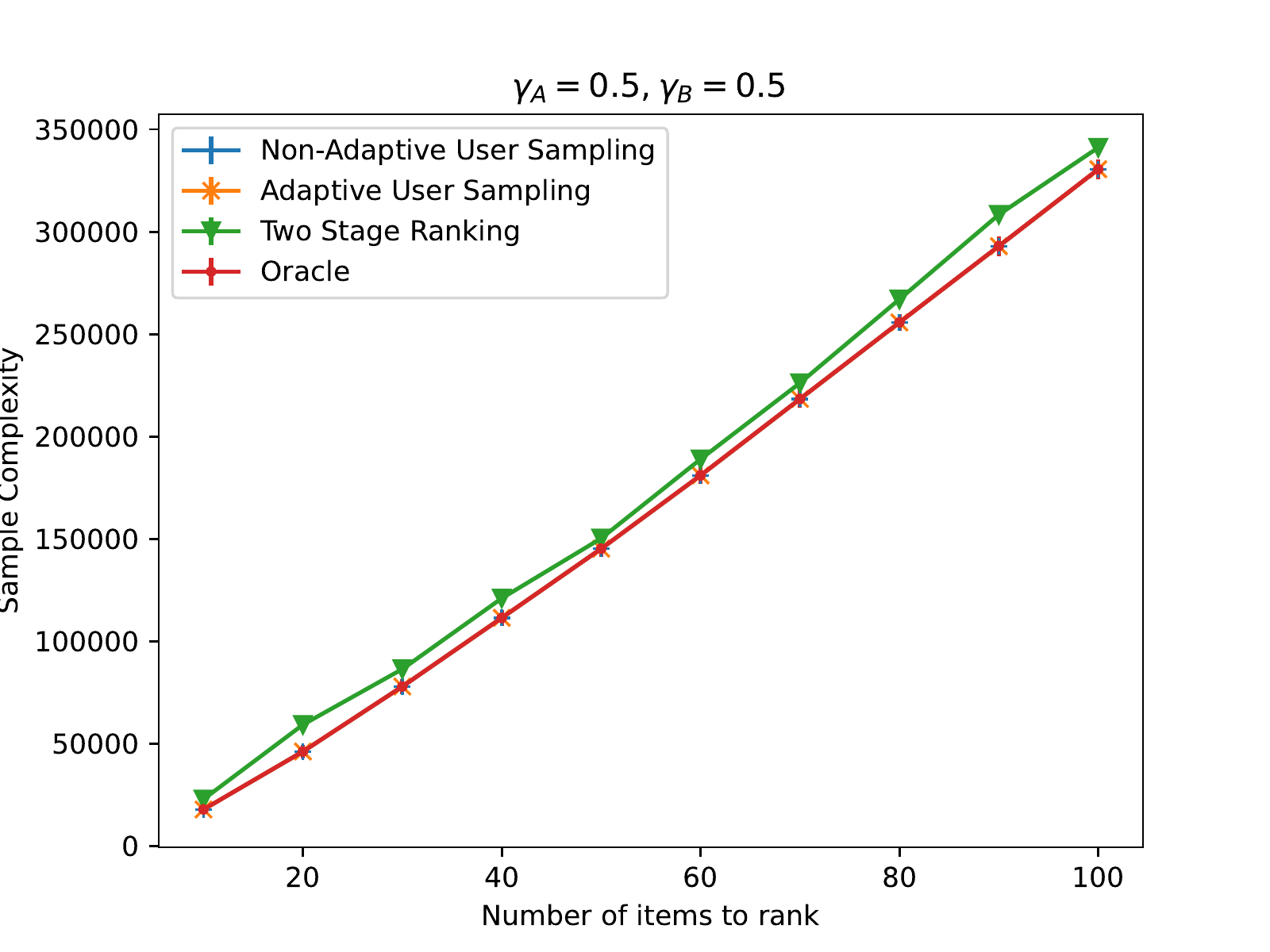} \label{fig:m9gb0.5gg0.5}}
\subfigure[$\gamma_A=0.5$, $\gamma_B=1.0$]{\includegraphics[width=0.32\textwidth]{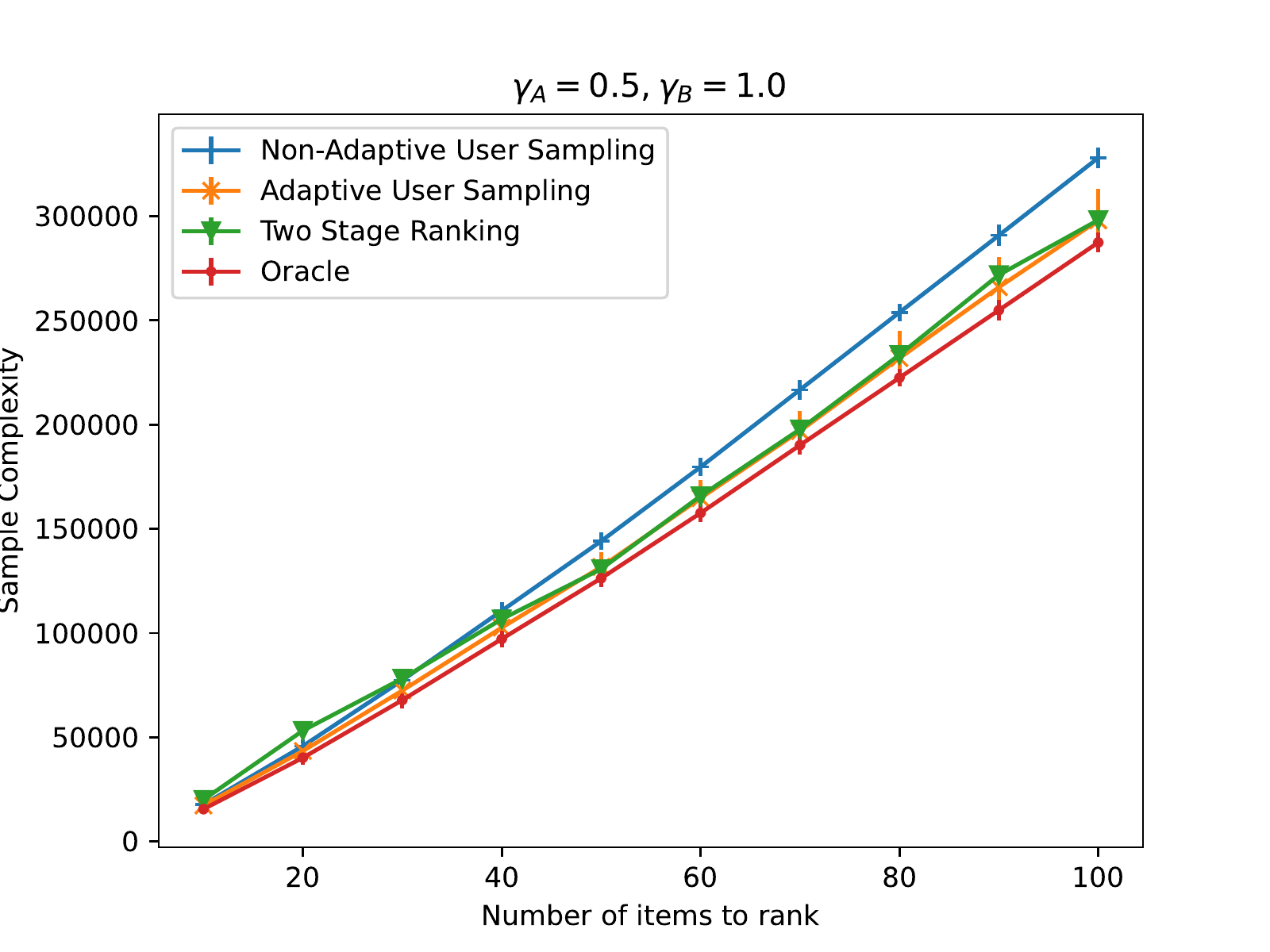} \label{fig:m9gb0.5gg1.0}}
\subfigure[$\gamma_A=0.5$, $\gamma_B=2.5$]{\includegraphics[width=0.32\textwidth]{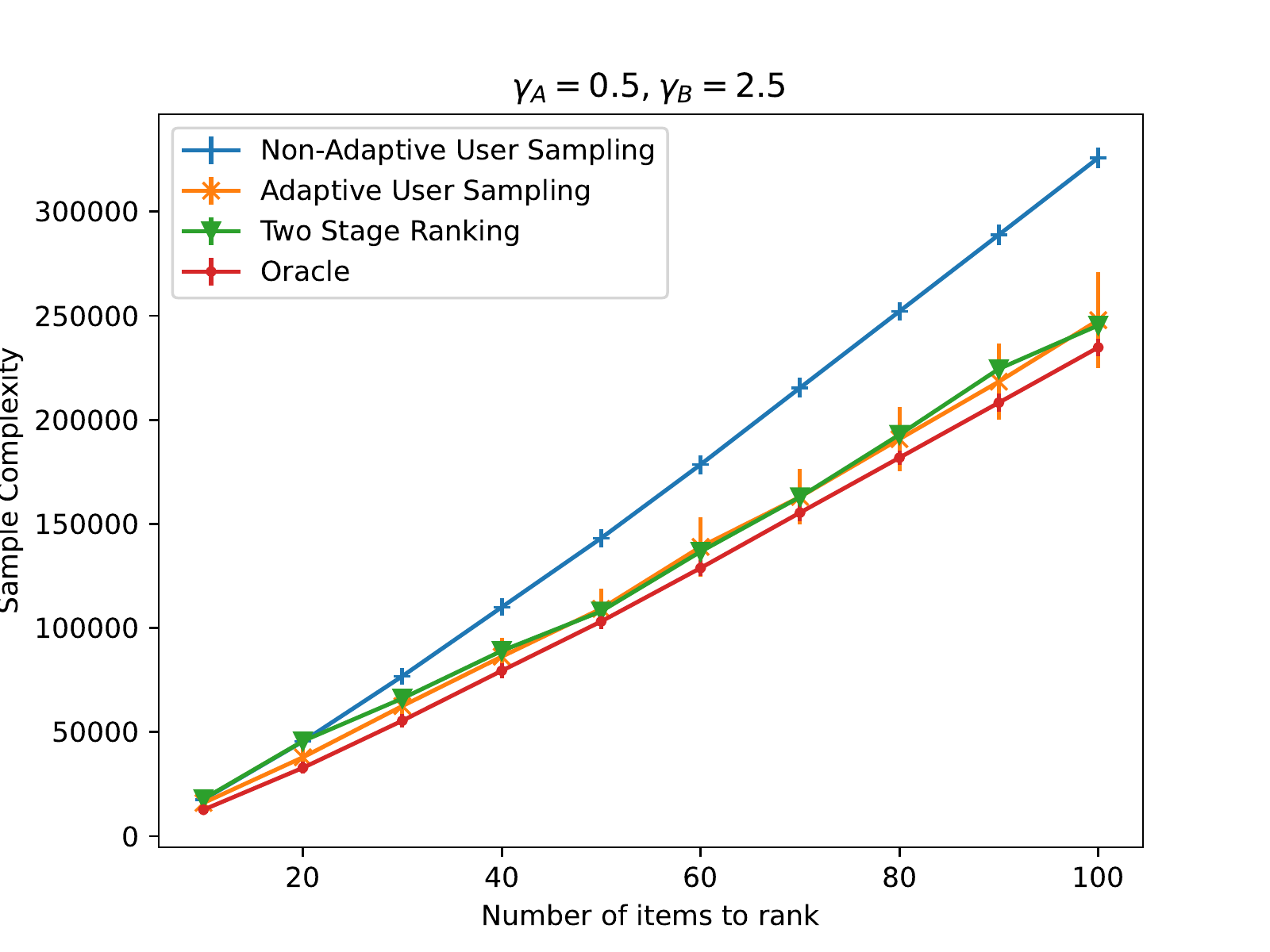} \label{fig:m9gb0.5gg2.5}}
\subfigure[$\gamma_A=1.0$, $\gamma_B=0.5$]{\includegraphics[width=0.32\textwidth]{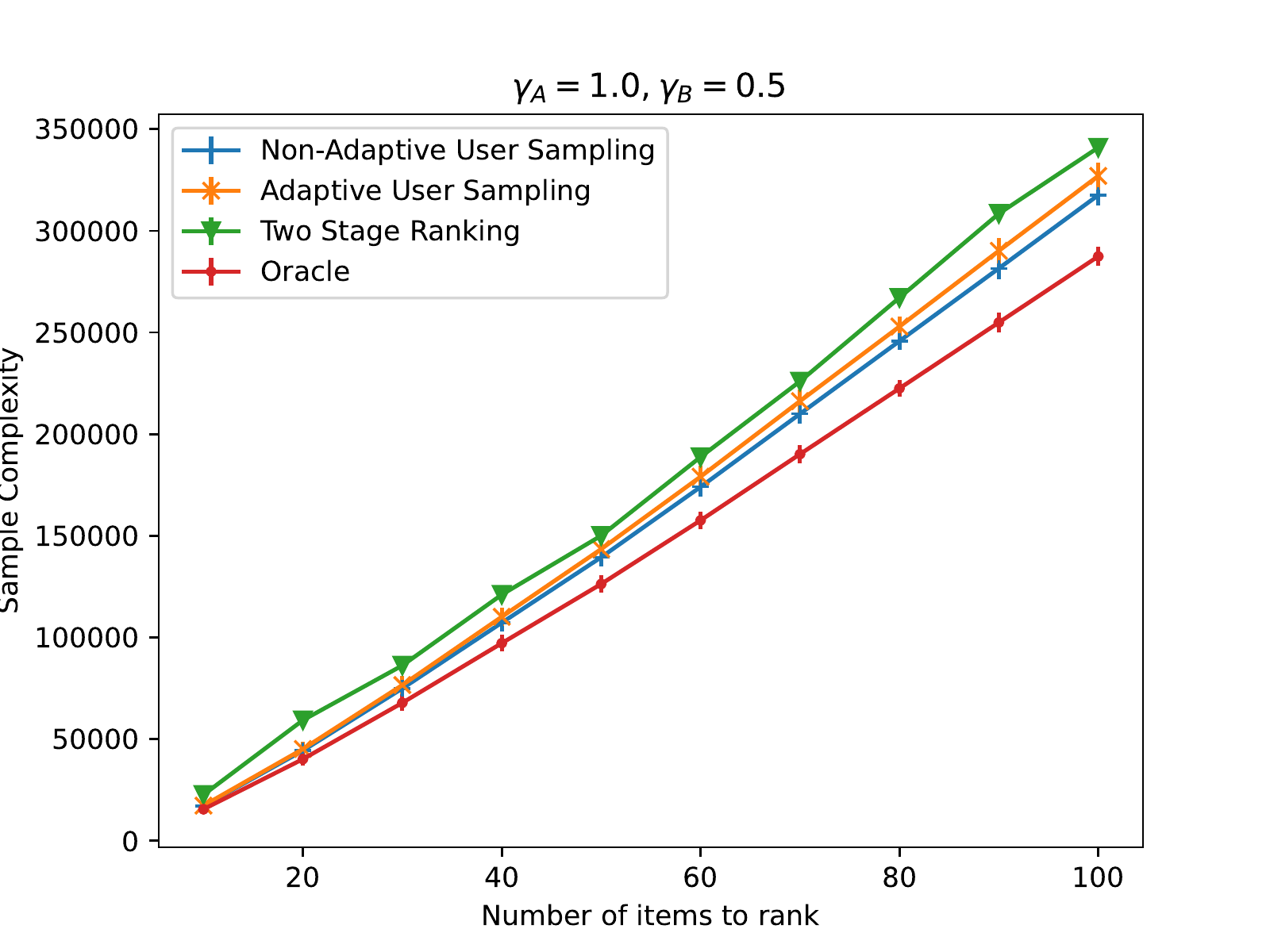} \label{fig:m9gb1.0gg0.5}}
\subfigure[$\gamma_A=1.0$, $\gamma_B=1.0$]{\includegraphics[width=0.32\textwidth]{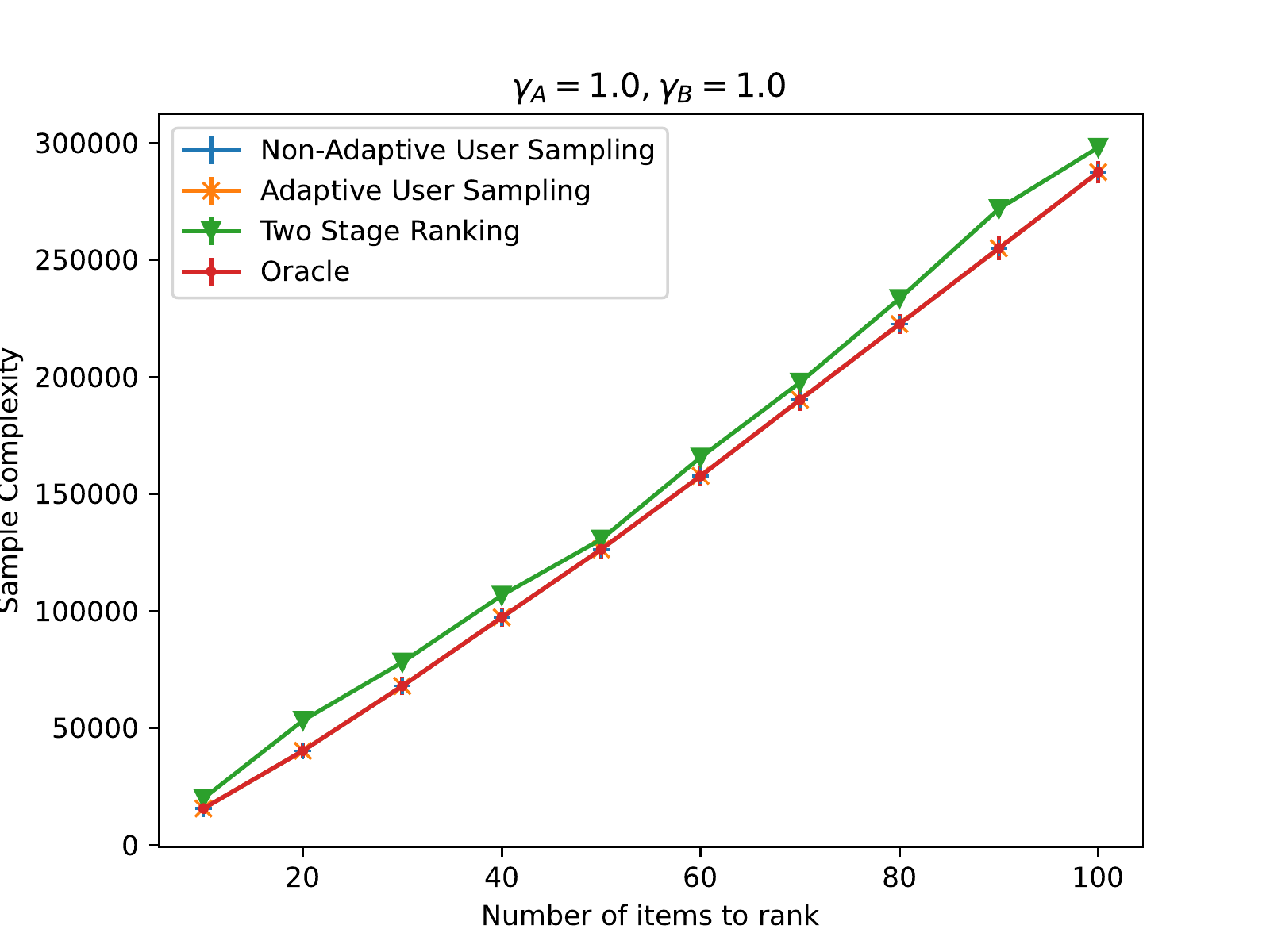} \label{fig:m9gb1.0gg1.0}}
\subfigure[$\gamma_A=1.0$, $\gamma_B=2.5$]{\includegraphics[width=0.32\textwidth]{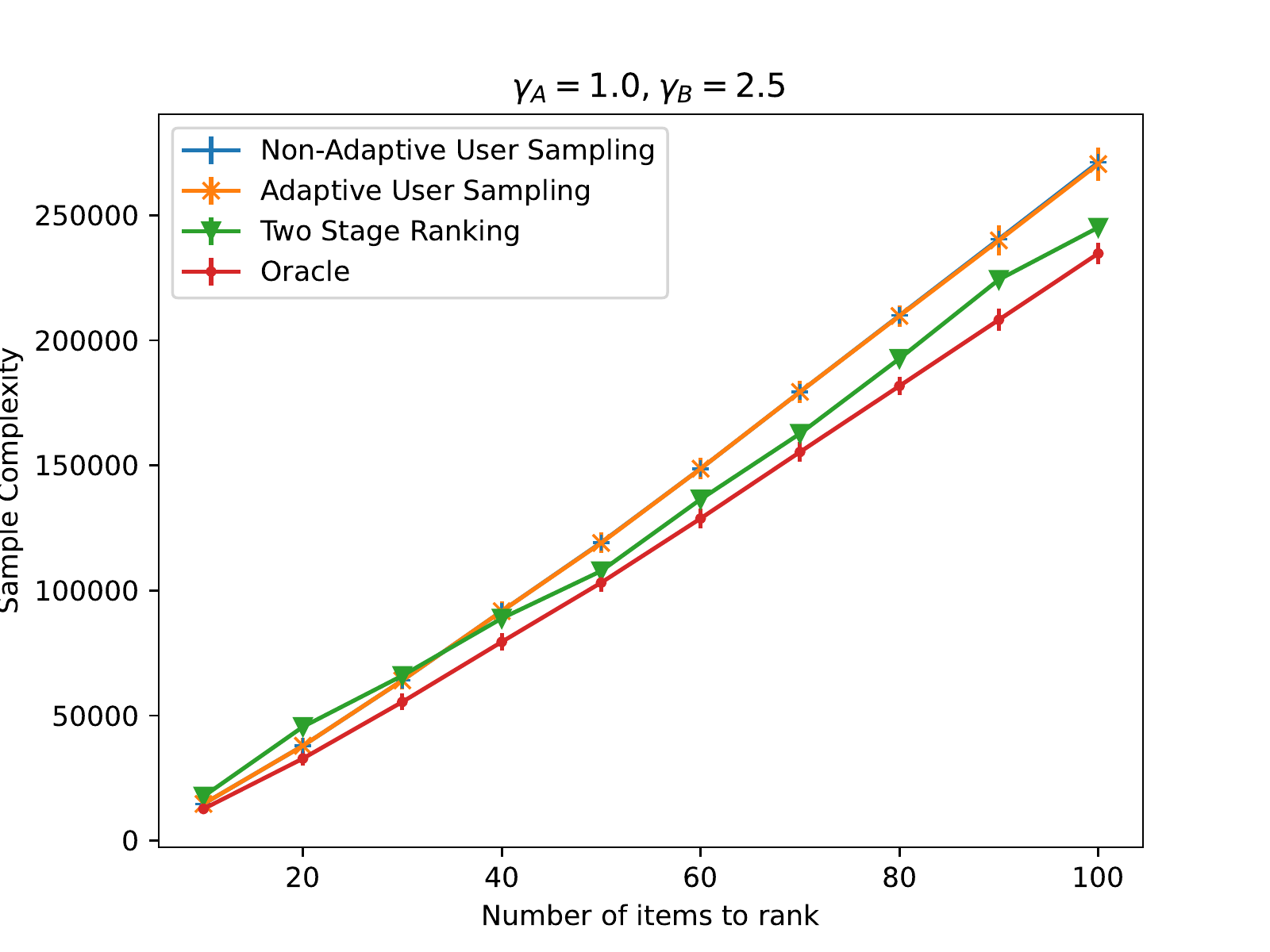} \label{fig:m9gb1.0gg2.5}}
            \caption{When $M = 9$. Sample complexities v.s. number of items for all algorithms. The 3-by-3 grid shows different heterogeneous user settings where the accuracy of two group of users differs.
            \label{fig:exp-m9}
            }
            \end{figure}
            
            \begin{figure}[ht!]
            \centering
            \subfigure[$\gamma_A=0.25$, $\gamma_B=0.5$]{\includegraphics[width=0.32\textwidth]{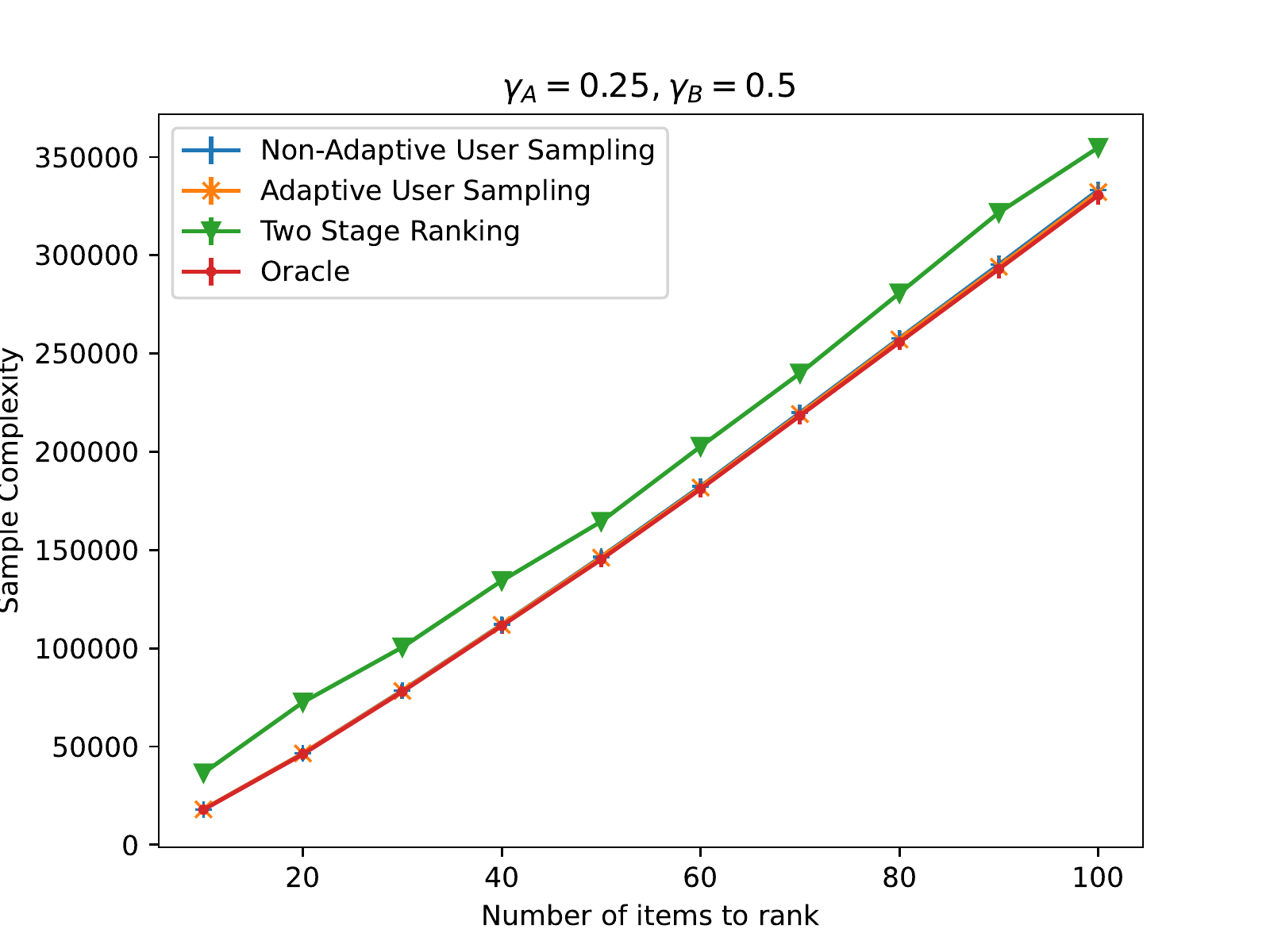} \label{fig:m18gb0.25gg0.5}}
\subfigure[$\gamma_A=0.25$, $\gamma_B=1.0$]{\includegraphics[width=0.32\textwidth]{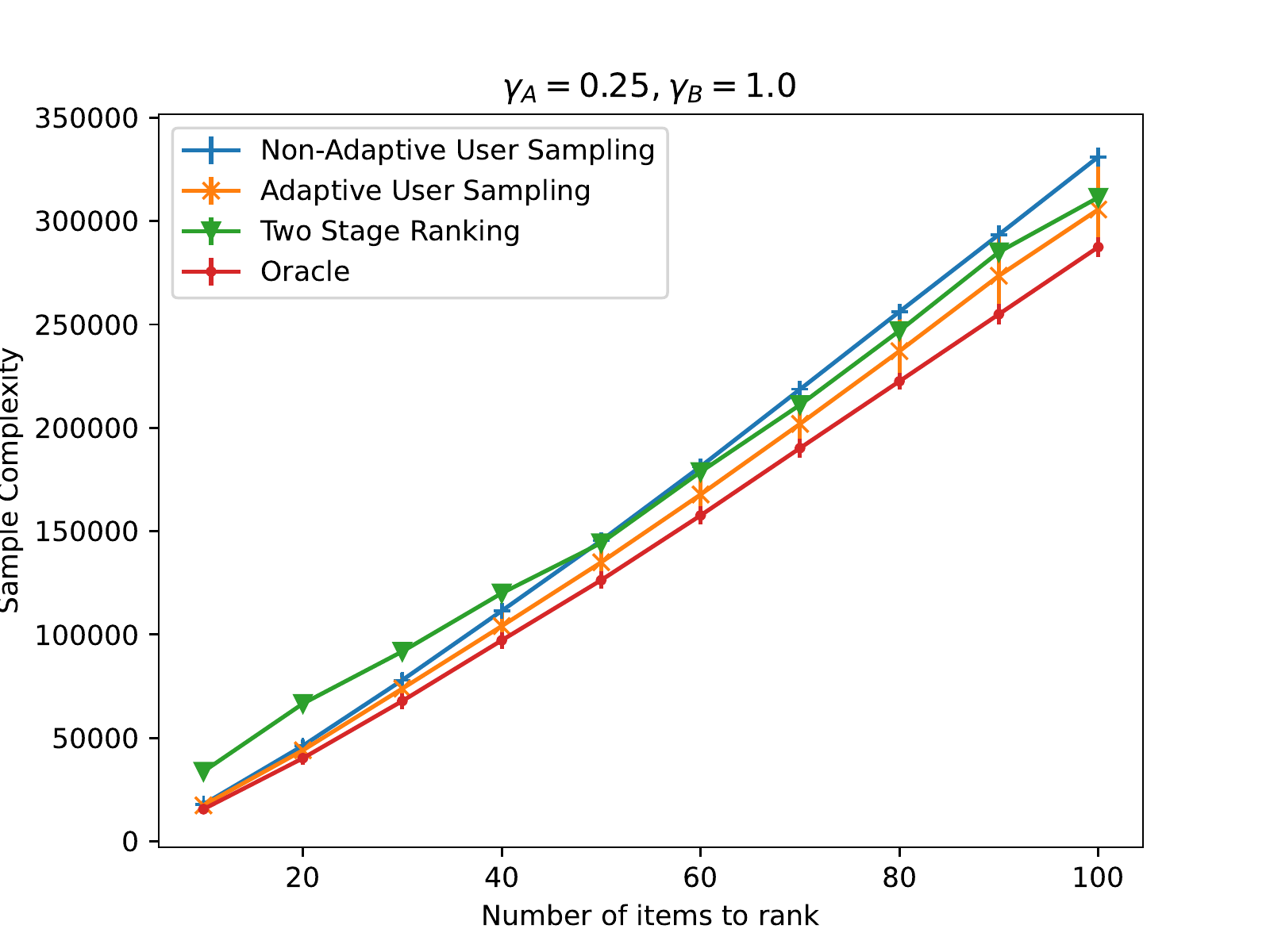} \label{fig:m18gb0.25gg1.0}}
\subfigure[$\gamma_A=0.25$, $\gamma_B=2.5$]{\includegraphics[width=0.32\textwidth]{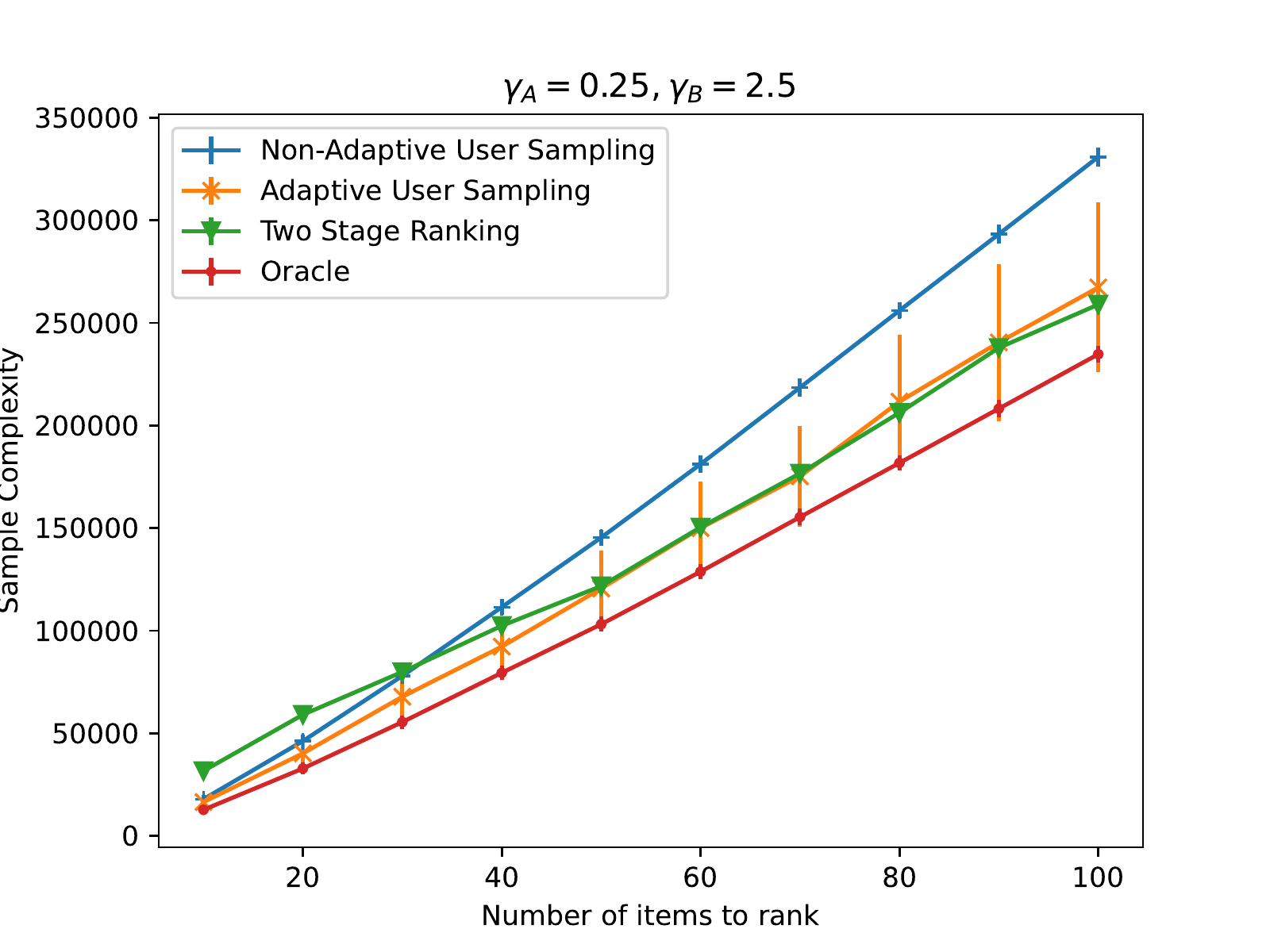} \label{fig:m18gb0.25gg2.5}}
\subfigure[$\gamma_A=0.5$, $\gamma_B=0.5$]{\includegraphics[width=0.32\textwidth]{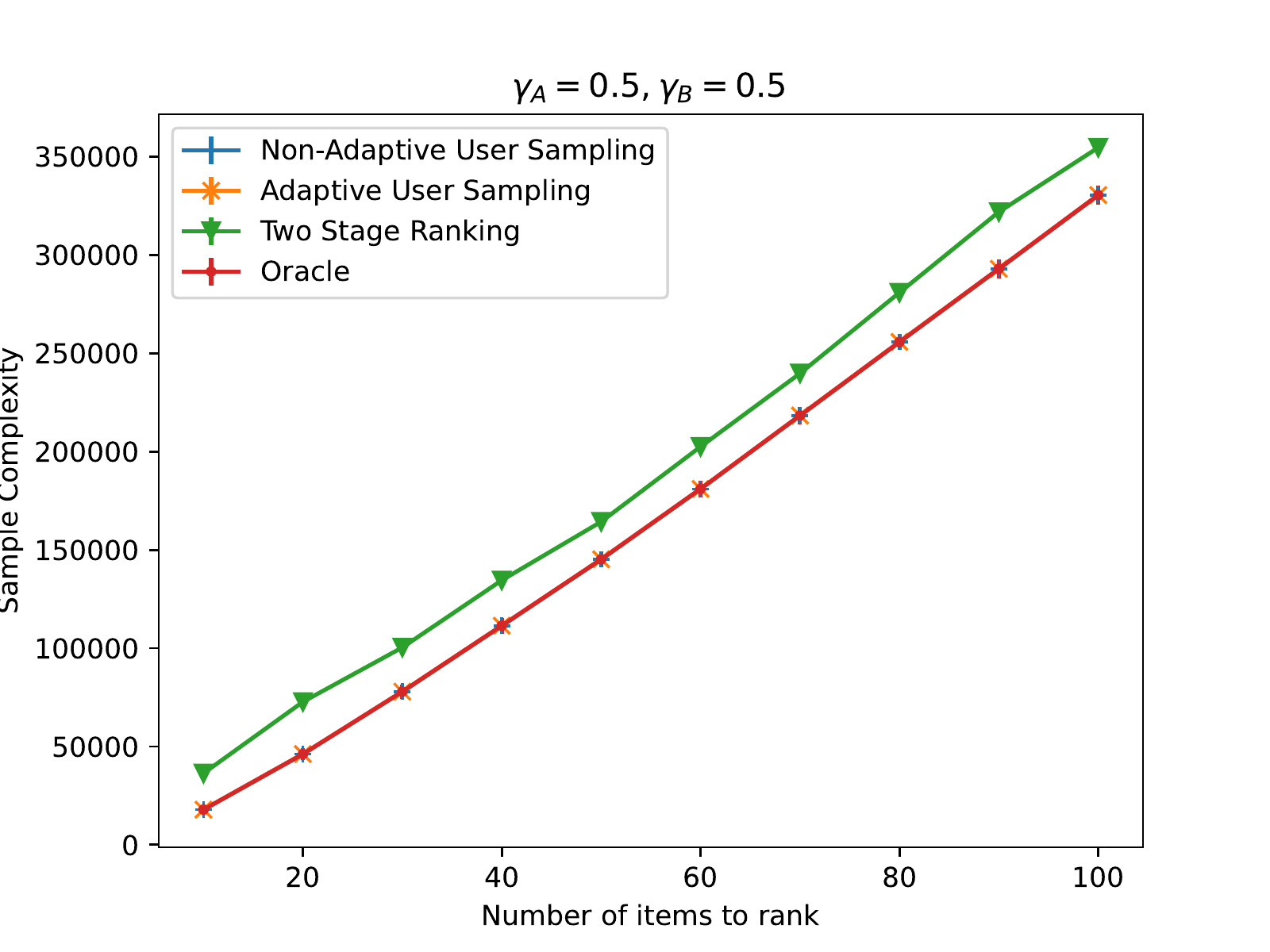} \label{fig:m18gb0.5gg0.5}}
\subfigure[$\gamma_A=0.5$, $\gamma_B=1.0$]{\includegraphics[width=0.32\textwidth]{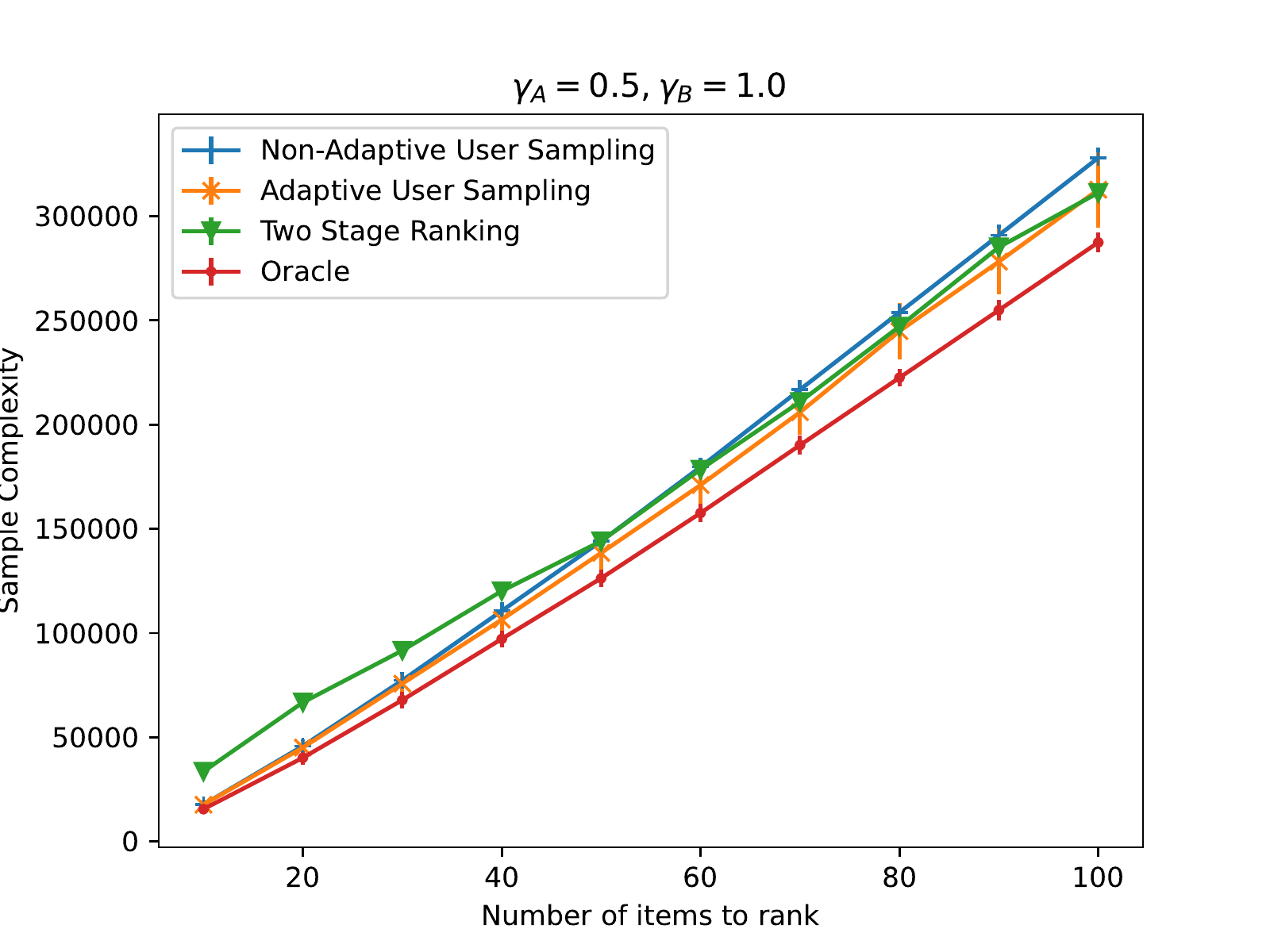} \label{fig:m18gb0.5gg1.0}}
\subfigure[$\gamma_A=0.5$, $\gamma_B=2.5$]{\includegraphics[width=0.32\textwidth]{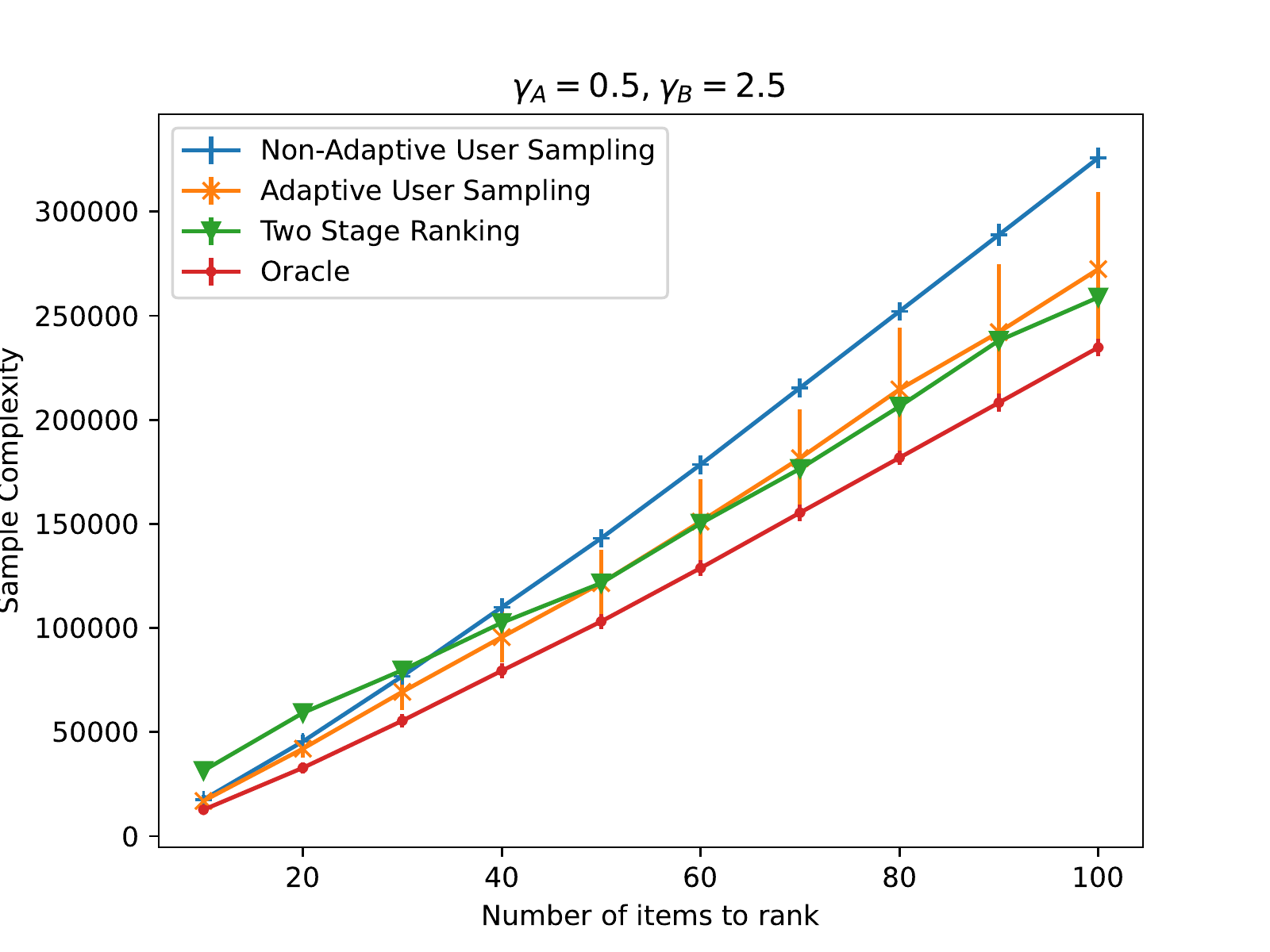} \label{fig:m18gb0.5gg2.5}}
\subfigure[$\gamma_A=1.0$, $\gamma_B=0.5$]{\includegraphics[width=0.32\textwidth]{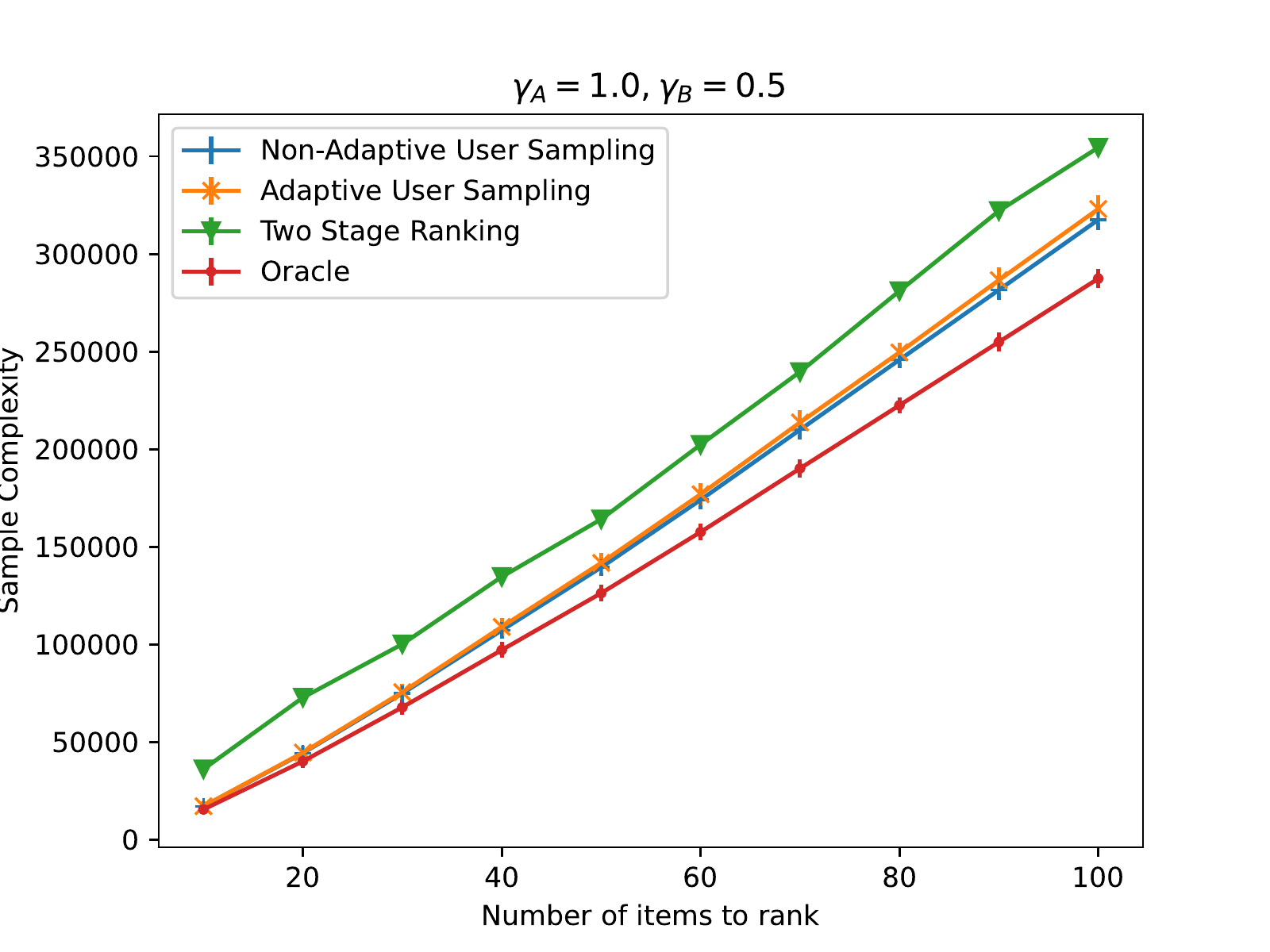} \label{fig:m18gb1.0gg0.5}}
\subfigure[$\gamma_A=1.0$, $\gamma_B=1.0$]{\includegraphics[width=0.32\textwidth]{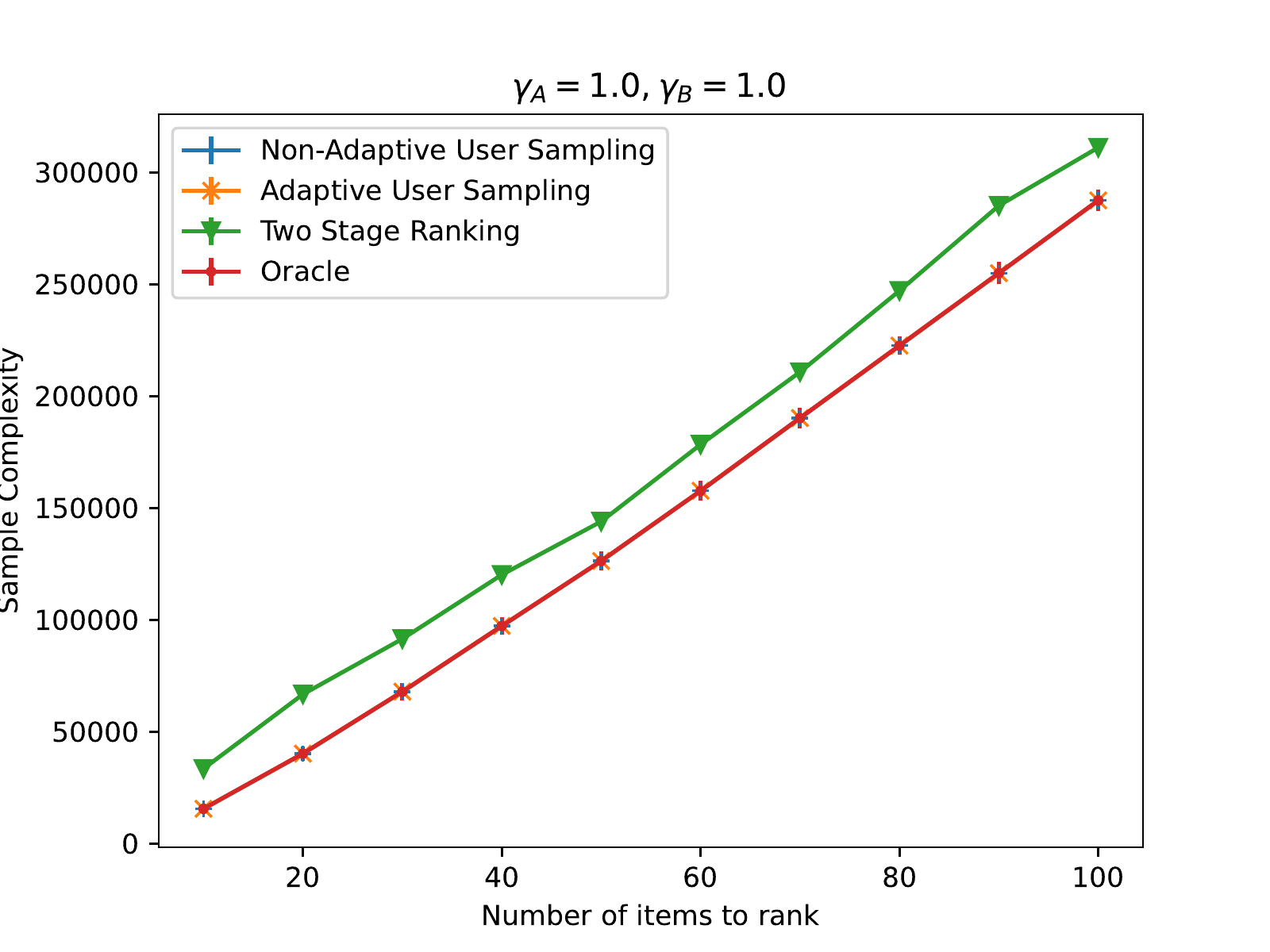} \label{fig:m18gb1.0gg1.0}}
\subfigure[$\gamma_A=1.0$, $\gamma_B=2.5$]{\includegraphics[width=0.32\textwidth]{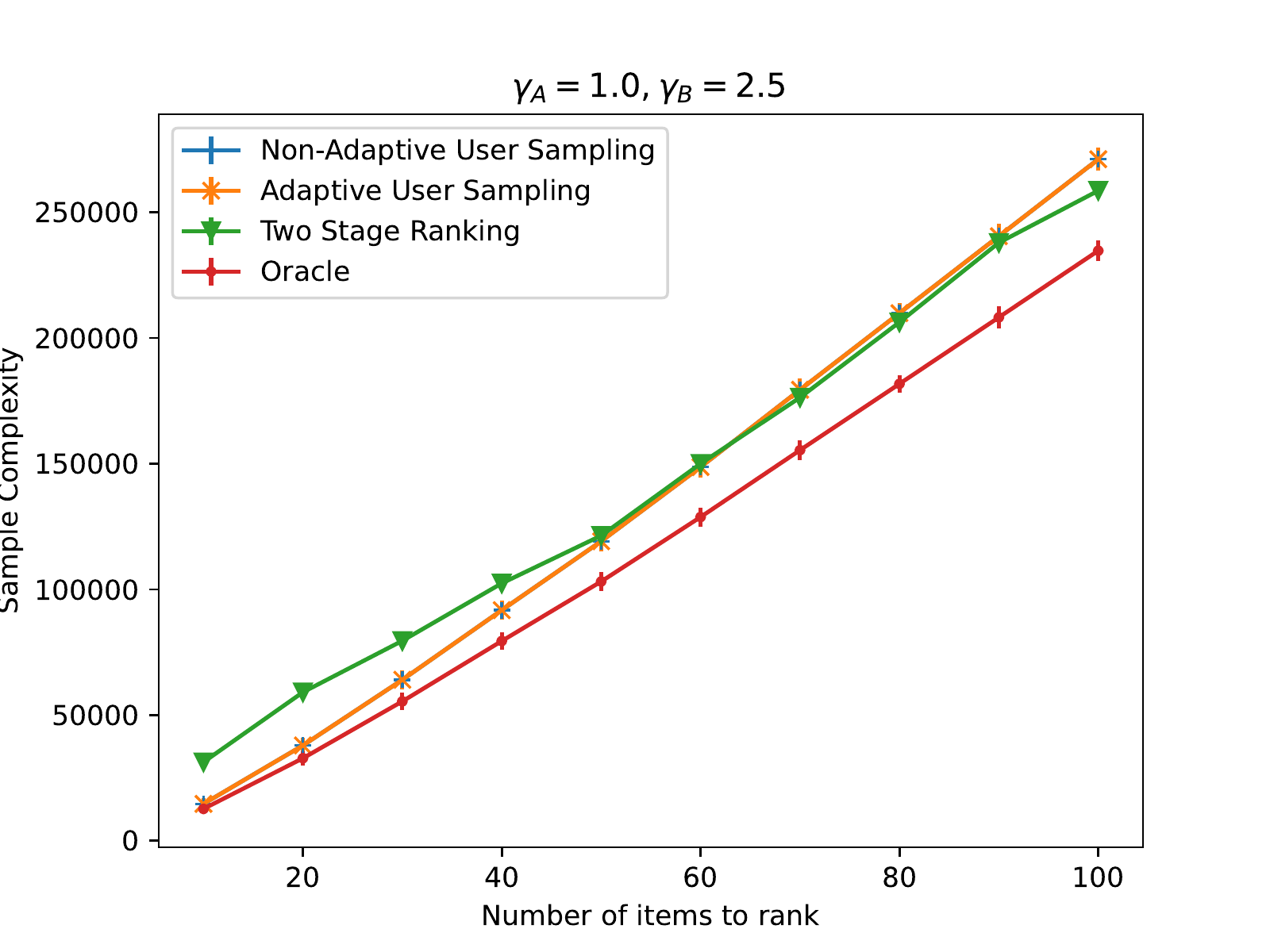} \label{fig:m18gb1.0gg2.5}}
            \caption{When $M = 18$. Sample complexities v.s. number of items for all algorithms. The 3-by-3 grid shows different heterogeneous user settings where the accuracy of two group of users differs.
            \label{fig:exp-m18}
            }
            \end{figure}
            
            \begin{figure}[ht!]
            \centering
            \subfigure[$\gamma_A=0.25$, $\gamma_B=0.5$]{\includegraphics[width=0.32\textwidth]{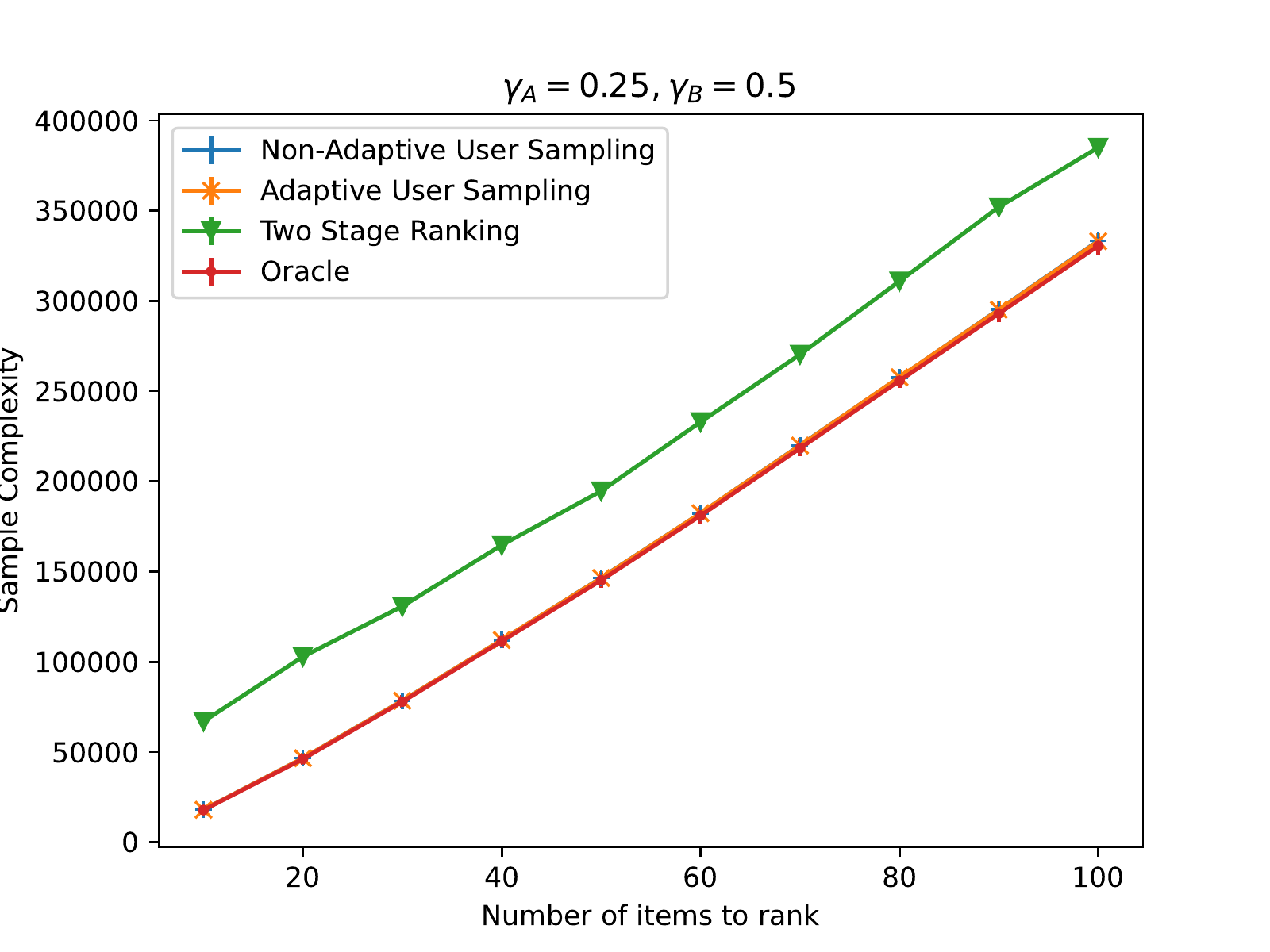} \label{fig:m36gb0.25gg0.5}}
\subfigure[$\gamma_A=0.25$, $\gamma_B=1.0$]{\includegraphics[width=0.32\textwidth]{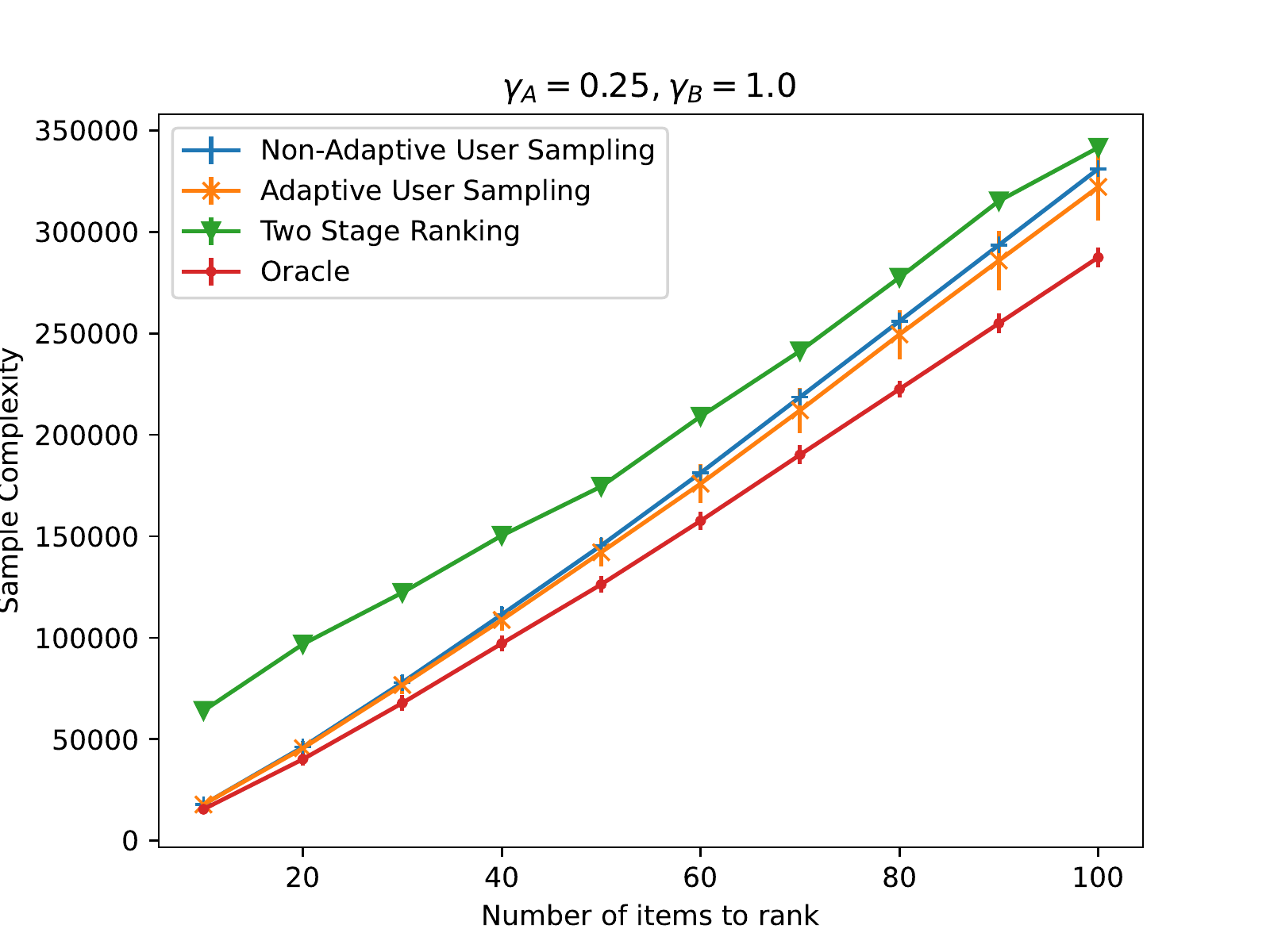} \label{fig:m36gb0.25gg1.0}}
\subfigure[$\gamma_A=0.25$, $\gamma_B=2.5$]{\includegraphics[width=0.32\textwidth]{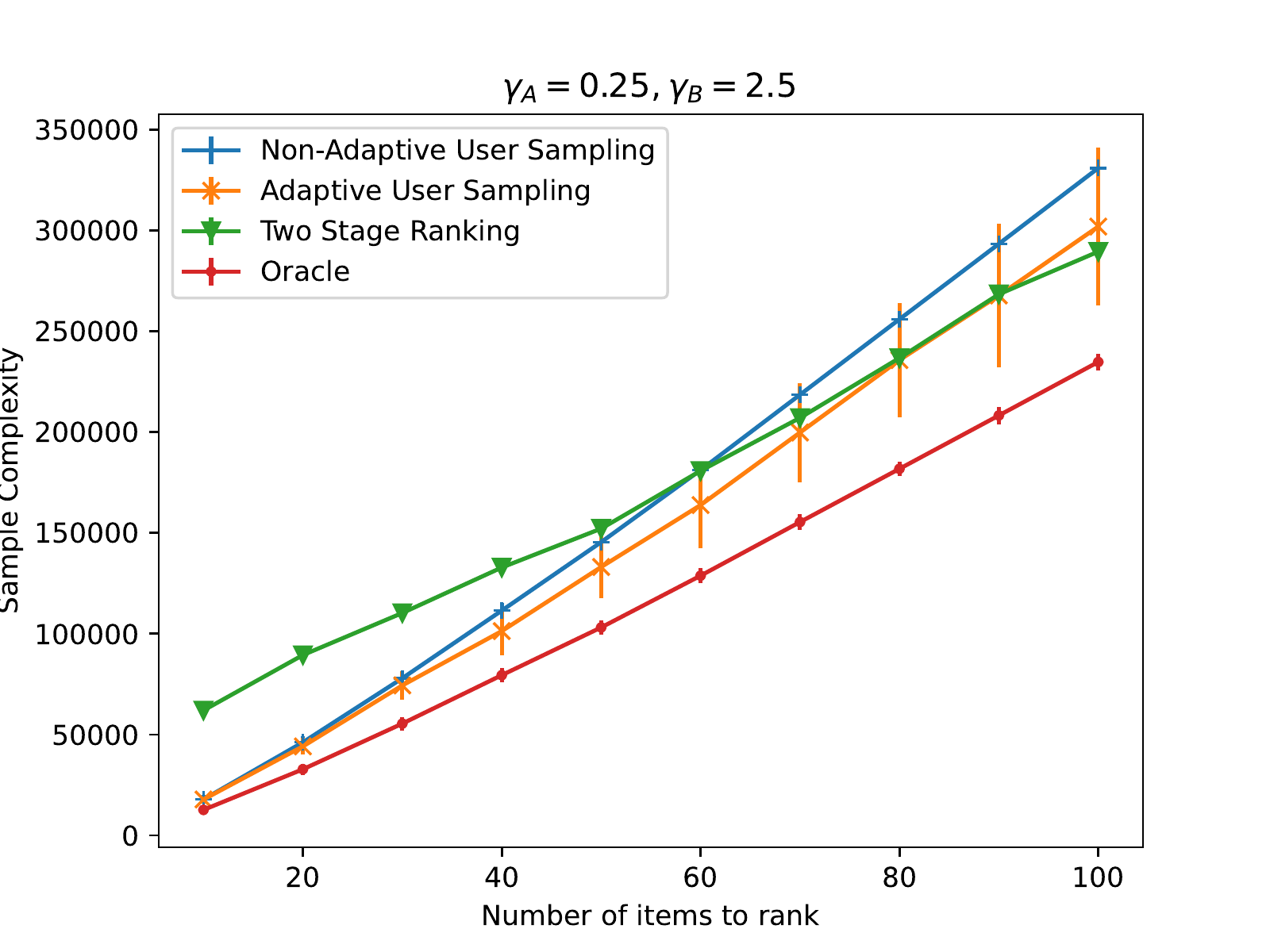} \label{fig:m36gb0.25gg2.5}}
\subfigure[$\gamma_A=0.5$, $\gamma_B=0.5$]{\includegraphics[width=0.32\textwidth]{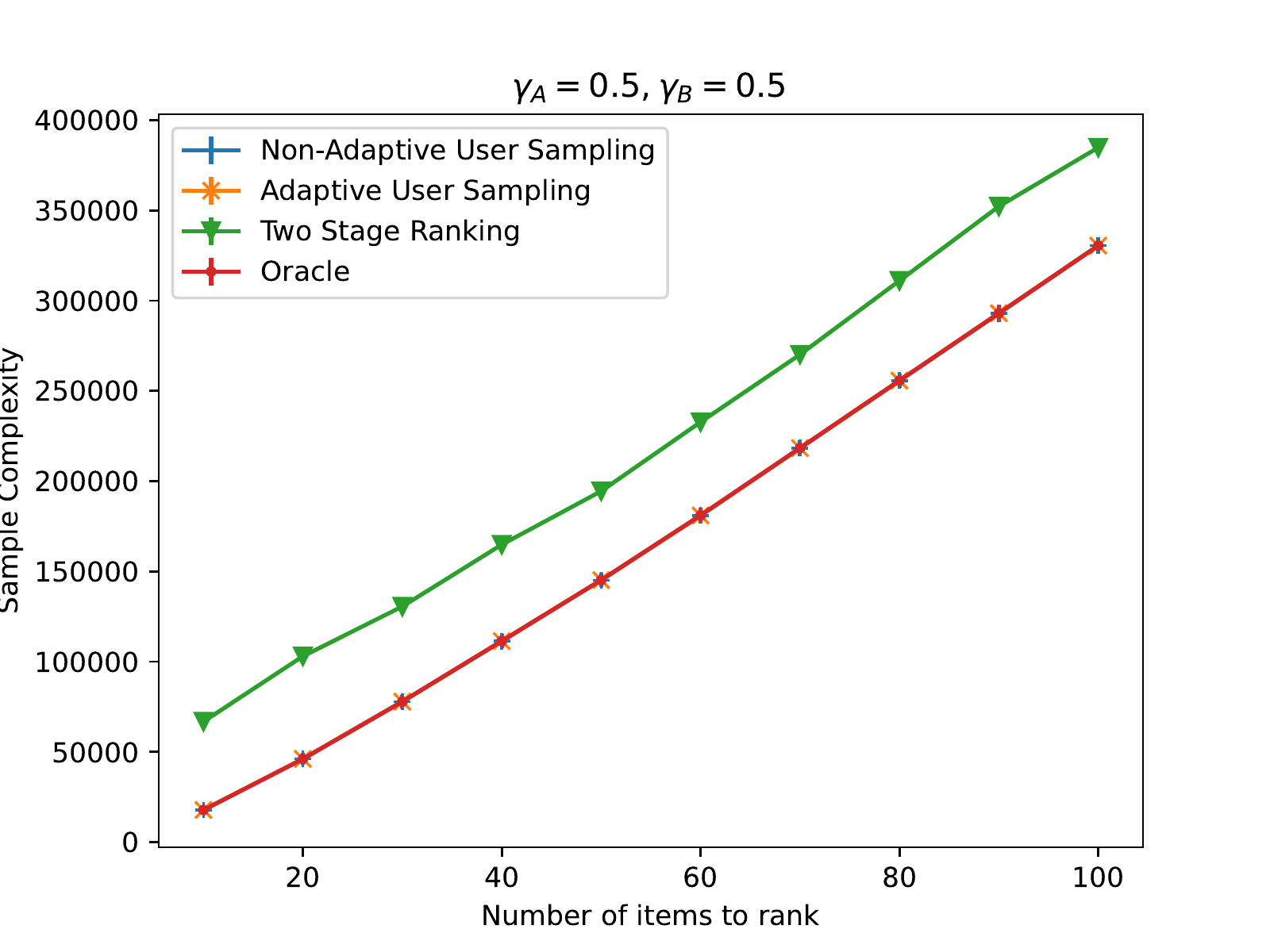} \label{fig:m36gb0.5gg0.5}}
\subfigure[$\gamma_A=0.5$, $\gamma_B=1.0$]{\includegraphics[width=0.32\textwidth]{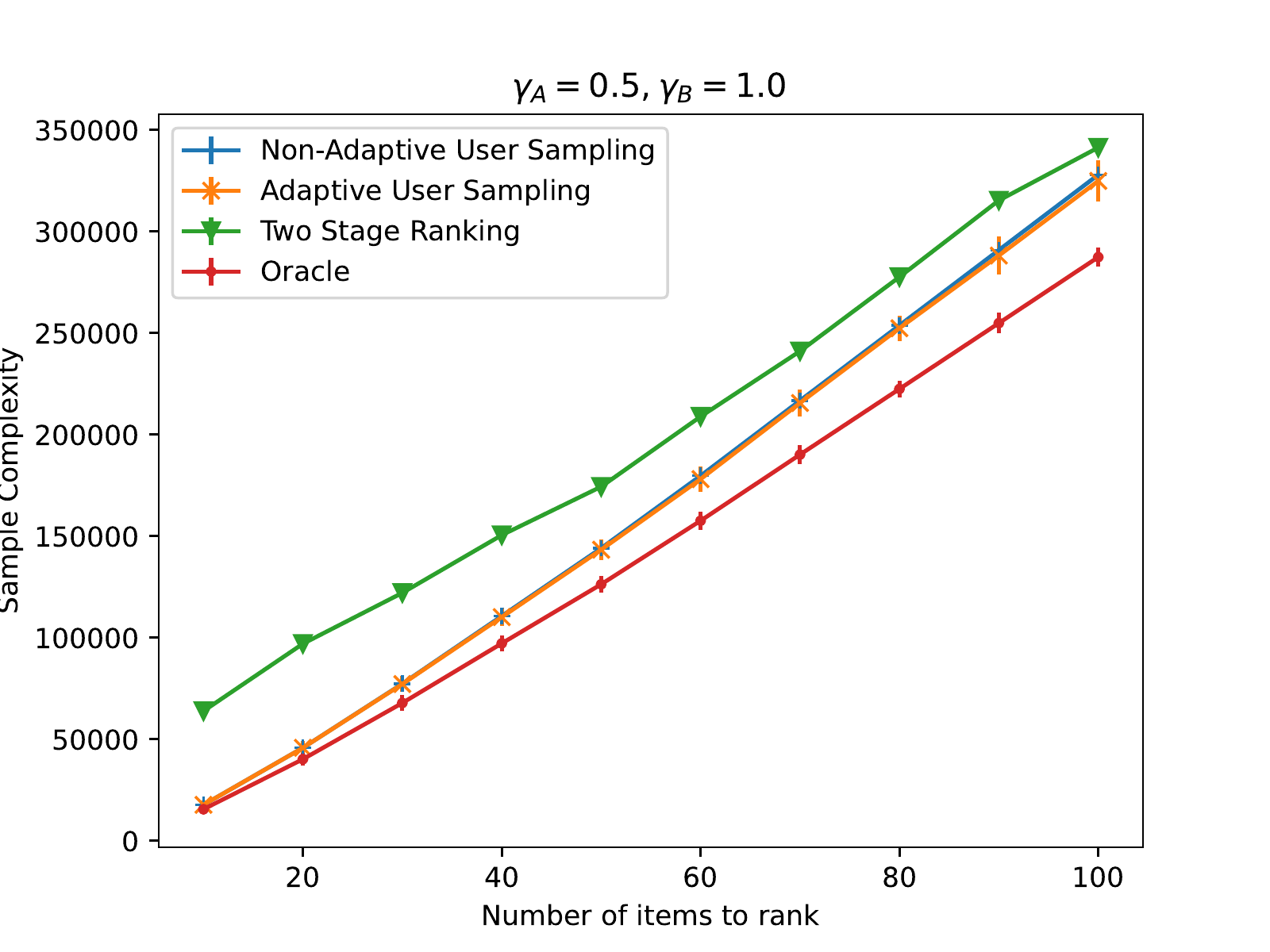} \label{fig:m36gb0.5gg1.0}}
\subfigure[$\gamma_A=0.5$, $\gamma_B=2.5$]{\includegraphics[width=0.32\textwidth]{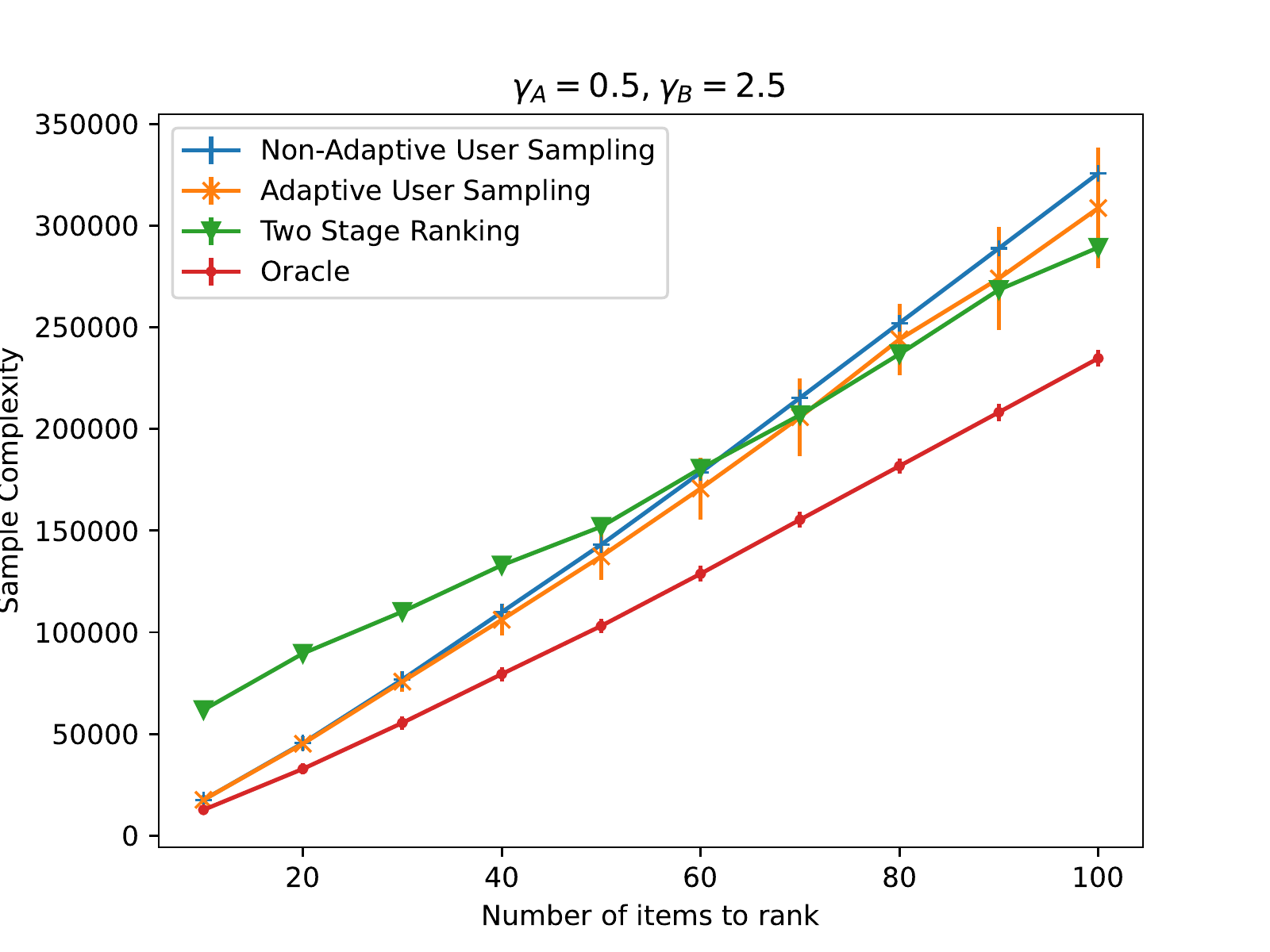} \label{fig:m36gb0.5gg2.5}}
\subfigure[$\gamma_A=1.0$, $\gamma_B=0.5$]{\includegraphics[width=0.32\textwidth]{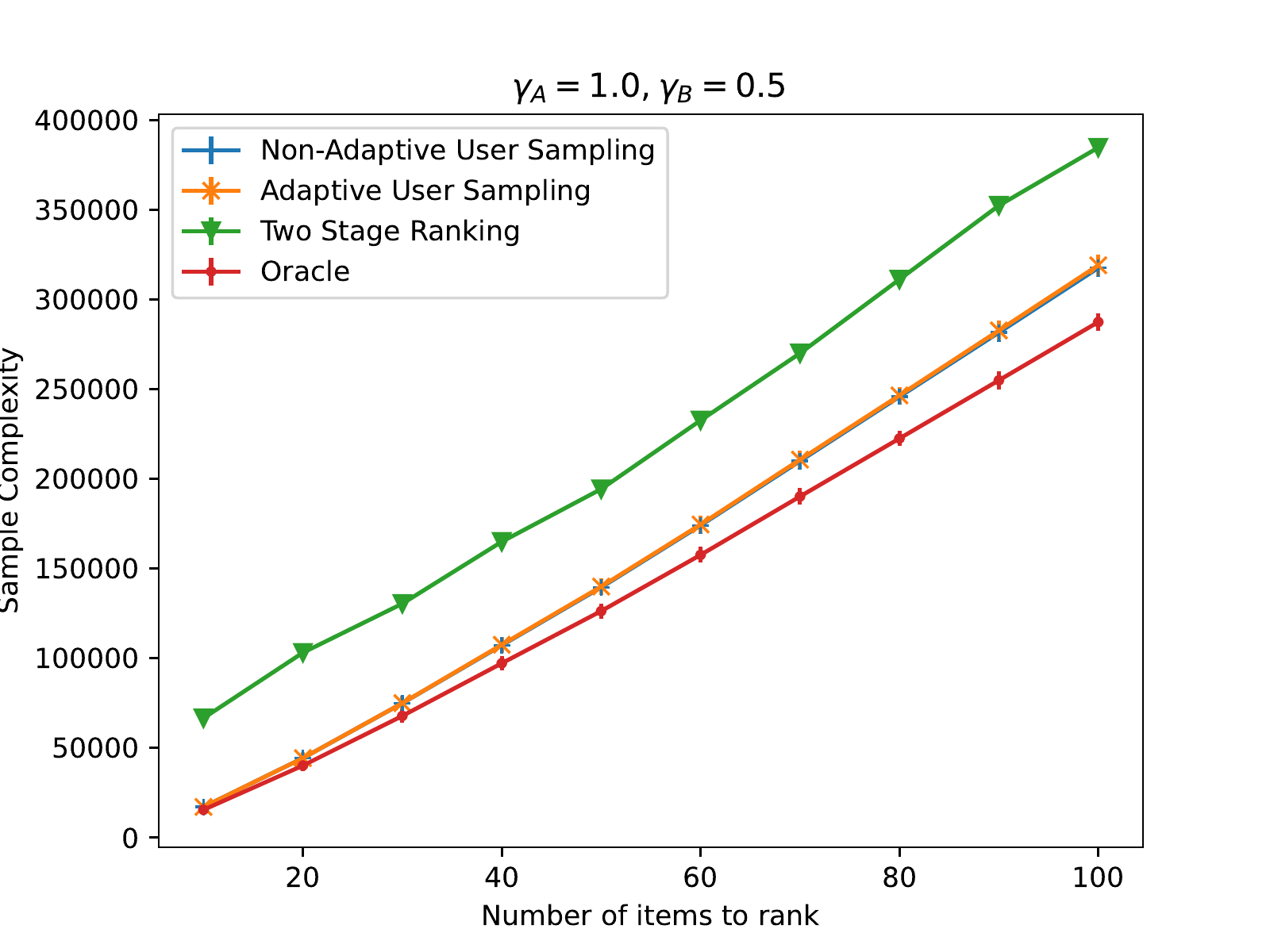} \label{fig:m36gb1.0gg0.5}}
\subfigure[$\gamma_A=1.0$, $\gamma_B=1.0$]{\includegraphics[width=0.32\textwidth]{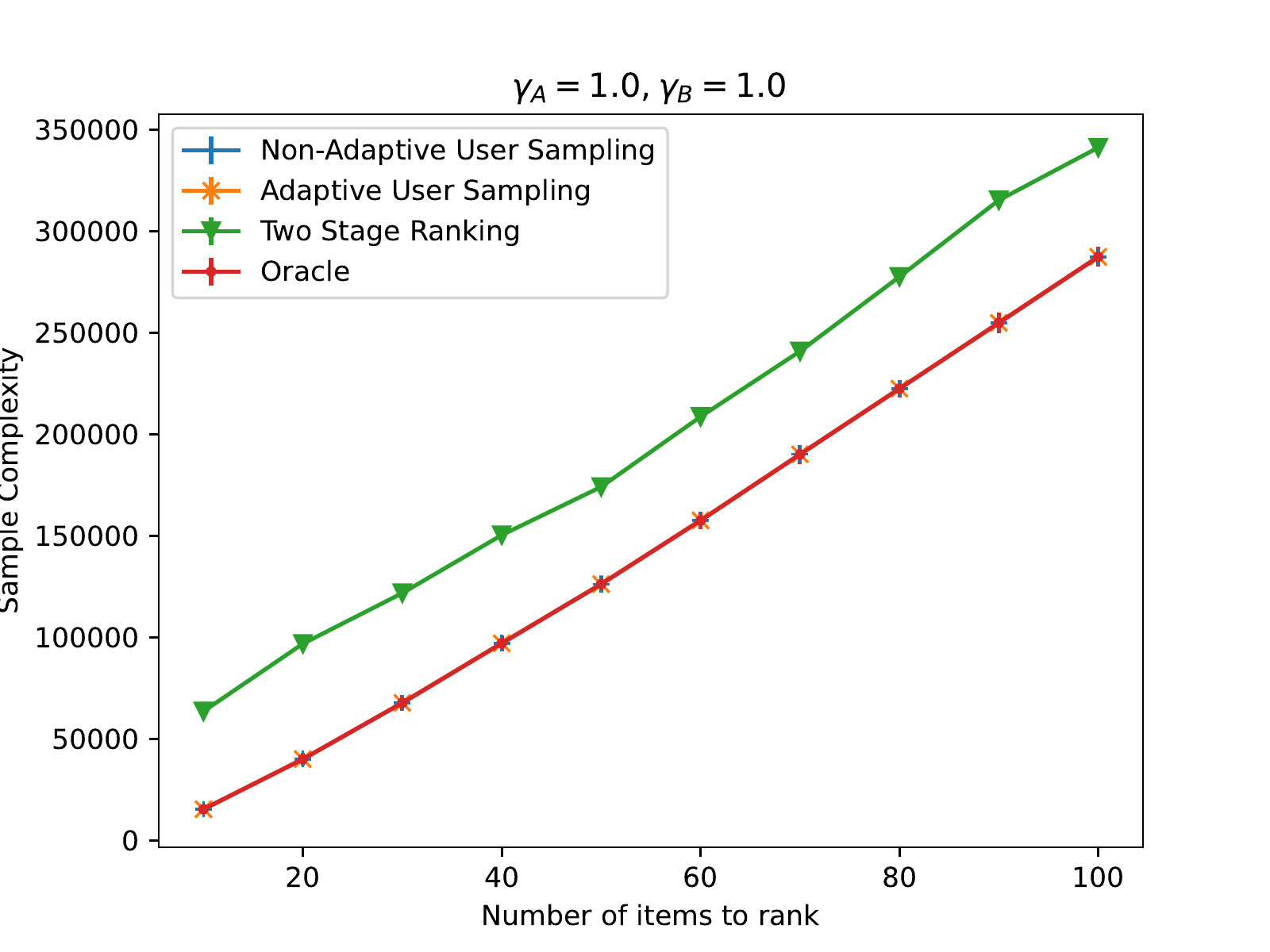} \label{fig:m36gb1.0gg1.0}}
\subfigure[$\gamma_A=1.0$, $\gamma_B=2.5$]{\includegraphics[width=0.32\textwidth]{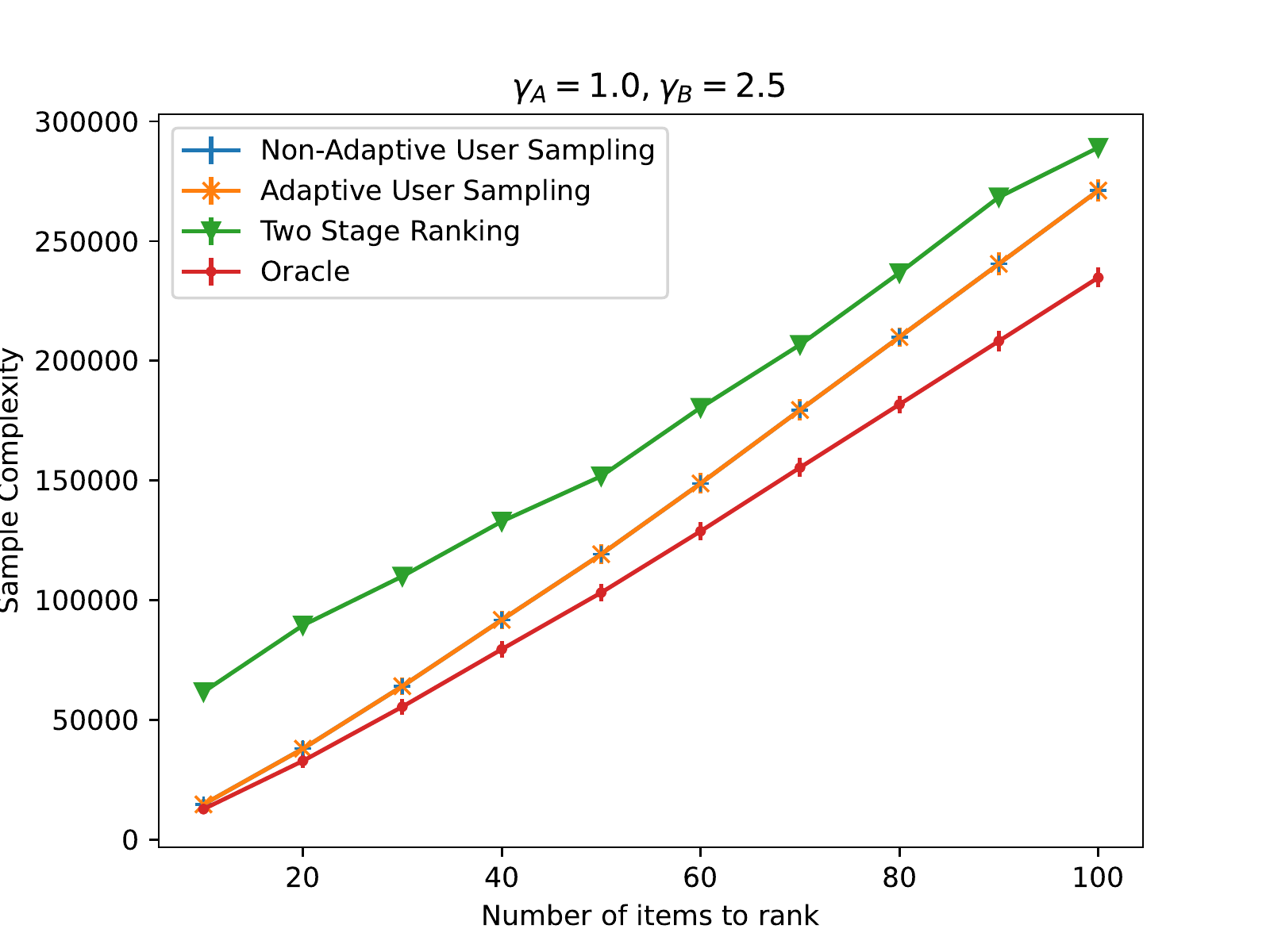} \label{fig:m36gb1.0gg2.5}}
            \caption{When $M = 36$. Sample complexities v.s. number of items for all algorithms. The 3-by-3 grid shows different heterogeneous user settings where the accuracy of two group of users differs.
            \label{fig:exp-m36}
            }
            \end{figure}

\section{Proof of Main Results}
\label{sec:proof-user-elim}
\subsection{Query Complexity of the Proposed Algorithm} \label{subsec:proof-complexity}
The following lemmas characterize the performance of each subroutine:

\begin{lemma}
[Lemma 9 in \citet{ren2019sample}]
\label{lem:atc}
    For any input pair $(i,j)$ and a set of users $\cU$, Algorithm \ref{alg:atc} terminates in $\lceil r_{\max} \rceil = \lceil \epsilon^{-2}\log(2/\delta) \rceil$ queries. 
    If $\epsilon \leq \bar{\Delta}$, then the returned $\hat{y}$ indicates the preferable item with probability at least $1 - \delta$. 
\end{lemma}

\begin{lemma}[Lemma 10 in \citet{ren2019sample}]
\label{lemma:ati}
Algorithm~\ref{alg:ati} returns after $O(\epsilon^2 \log(|S|/\delta)$ queries and, with probability $1-\delta$, correctly insert or return unsure. Additionally, if $\epsilon \le \bar{\Delta}$, Algorithm~\ref{alg:ati} will insert correctly with probability $1-\delta$.
\end{lemma}

\begin{lemma}[Lemma 11 in \citet{ren2019sample}]
\label{lemma:iai}
With probability $1-\delta$,
Algorithm~\ref{alg:iai} correctly insert the item and makes $O(\Bar{\Delta}^{-2}
    (
    \log \log \Bar{\Delta}^{-1}
    +
    \log(N/\delta)
    ) )$ queries at most.
\end{lemma}

\begin{proof}[Proof of Theorem~\ref{theorem:q-complexity}]

When inserting the $z$-th item, we makes at most $\Bar{\Delta}_z^{-2}
    (
    \log \log \Bar{\Delta}_z^{-1}
    +
    \log(N/\delta)
    )$ queries, for $z=2,3,\dots,N$.

The number of total queries can be obtained by summing up the term above, which is 
\begin{align*}
    \cC_{\mathrm{Alg}}(N)
    & = 
    O\bigg(
    \sum_{z=2}^{N}
    \bar{\Delta}_z^{-2}  (\log\log\bar{\Delta}_z^{-1} + \log(N/\delta)) 
    \bigg).
\end{align*}

\end{proof}

\subsection{Complexity Gap Analysis}
\label{subsec:proof-comp-gap}
The first lemma we will introduce is about the confidence interval:
\begin{lemma} \label{lemma:conf-int}
With probability $1-\delta$, it holds for any $z \in [N] \backslash \{1\}$ and $u \in \cU_z$,
\begin{align*}
    \frac{1}{2} + \Delta_u 
    & \in 
    \Big[
    (\mathbf{{LCB}}_z)_u, (\mathbf{{UCB}}_z)_u
    \Big].
\end{align*}
This also indicates that when inserting the $z$-th item, for any $u \in \cU_z$,
\begin{align*}
    \Delta_{u^*} - \Delta_{u}
    & \le 
    4 r_z.
\end{align*}
\end{lemma}
\begin{proof}[Proof of Lemma~\ref{lemma:conf-int}]

Recall that $(\bmu_z)_u$ is the empirical mean of the Bernoulli variable with parameter $\frac{1}{2}+ \Delta_u$. For a given $z$ and $u$, by Hoeffding's inequality we have
\begin{align*}
    \PP
    \bigg(
    \Big|
    (\bmu_{z})_u
    -
    \Big(
    \frac{1}{2} + \Delta_{u}
    \Big)
    \Big|
    >
    r_z
    \bigg)
    \le 
    2e^{-2 (\sbb_z)_u r_u^2}
    \le 
    2e^{-2 (\sbb_z)_{\min} r_u^2}
    \le 
    \frac{\delta}{|\cU_z|N},
\end{align*}
and applying union bound over $z =2,3,\dots,N$ and $u \in \cU_z$ gives the claim.

Under this event, we have
\begin{align*}
    \Delta_{u^*} - {\Delta}_{u}
    & =
    \bigg(
    \frac{1}{2} + \Delta_{u^*}
    \bigg)
    -
    \bigg(
    \frac{1}{2} + \Delta_{u}
    \bigg)
    \\
    & \le 
    (\mathbf{UCB}_z)_{u^*}
    -
    (\mathbf{LCB}_z)_{u}
    \\
    & \le 
    (\mathbf{UCB}_z)_{u^*}
    -
    (\mathbf{LCB}_z)_{u^*}
    +
    (\mathbf{UCB}_z)_{u}
    -
    (\mathbf{LCB}_z)_{u}
    \\
    & = 4 r_z,
\end{align*}
where the first inequality is clearly from the confidence interval, and the second inequality holds because the two confidence intervals should intersect.
\end{proof}

Next, we will introduce another lemma concerning the growth of $(\sbb_z)_u$ for each $u \in \cU_z$.
\begin{lemma} \label{lemma:query-per-user}
Denote $S_z$ as all queries made till inserting the $z$-th item and $M=|\cU_0|$.
Suppose $S_z \ge 2M^2 \log(NM/\delta)$. With probability $1-\delta$, we have for any $z \in \{2,3,\dots,N\}$,
\begin{align*}
    (\sbb_z)_{\min}
    \ge 
    \frac{S_z}{2M}.
\end{align*}
\end{lemma}
\begin{proof}[Proof of Lemma~\ref{lemma:query-per-user}]

For fixed $z$ and $u \in \cU_z$, by Hoeffding's inequality we have
\begin{align*}
    \PP\bigg(
    \frac{(\sbb_z)_u}{S_z}
    - \frac{1}{M}
    < -\frac{1}{2M}
    \bigg)
    & \le
    \PP\bigg(
    \frac{(\sbb_z)_u}{S_z}
    - 
    \EE \bigg[ \frac{(\sbb_z)_u}{S_z}
    \bigg]
    < -\frac{1}{2M}
    \bigg)
    \\
    & \le 
    \exp 
    \bigg(
    -\frac{S_z}{2M^2}
    \bigg)
    \le 
    \frac{\delta}{NM}.
\end{align*}
Applying union bound we know that with probability $1-\delta$, 
\begin{align*}
    (\sbb_z)_u \ge \frac{S_z}{2M}, \forall z \in \{2,3,\dots,N\}, \forall u \in \cU_z.
\end{align*}
Since $(\sbb_z)_{\min} := \min_{u \in \cU_z} (\sbb_z)_u$, we have
\begin{align*}
    (\sbb_z)_{\min}
    \ge 
    \frac{S_z}{2M}, \forall z \in \{2,3,\dots,N\}.
\end{align*}
\end{proof}

With the two lemmas above, we can control the accuracy gap as follows:
\begin{lemma} \label{lemma:accuracy-gap}
Denote $\bar{\Delta}_z = \frac{1}{| \cU_z |} \sum_{u \in \cU_z} 
\Delta_{u}$.
Suppose $S_z \ge 2 |M|^2 \log(NM/\delta)$. With probability $1-2\delta$, we have for any $t \in [N]$,
\begin{align*}
    \Delta_{u^*} - \bar{\Delta}_{z}
    & \le 
    \mathrm{polylog}(N, M, \delta^{-1}) 
    \cdot 
    \sqrt{\frac{M}{S_z}}.
\end{align*}
\end{lemma}
\begin{proof}[Proof of Lemma~\ref{lemma:accuracy-gap}]
The proof has two steps: 

From Lemma~\ref{lemma:query-per-user} we know that with probability $1-\delta$,
\begin{align*}
    (\sbb_z)_{\min} \ge \frac{S_z}{2M}, \forall t \in [N], \forall u \in \cU_z.
\end{align*}

From Lemma~\ref{lemma:conf-int}, we know with probability $1-\delta$(recall that $(\rb_z)_u = \sqrt{\frac{\log(2|\cU_z|N / \delta)}{2 (\sbb_z)_{\min}}}$),
\begin{align*}
    \Delta_{u^*} - {\Delta}_{u}
    & \le 
    4 r_z
    \\
    &
    \le 
    4 
    \sqrt{\frac{M\log(2MN / \delta)}{S_z}}
    \\
    & =
    4 \sqrt{\log(2MN/\delta)}
    \cdot 
    \sqrt{\frac{M}{S_z}}.
\end{align*}
\end{proof}

Define function $F(x) = x^{-2}(\log\log(x^{-1}) + \log(N/\delta))$ with $x \in (0, 1/2]$. We care about the following term $\mathrm{GAP}$ which characterize the query complexity gap between our algorithm and the optimal user.
\begin{align*}
    \mathrm{GAP}(N, M, \delta)
    & =
    \sum_{z=2}^{N}
    F(\bar{\Delta}_z)
    -
    F(\Delta_{u^*}).
\end{align*}

The following lemma provide a way to linear bound the gap between function values:




\begin{lemma} \label{lemma:function-F}
$F(x) = x^{-2}(\log\log(x^{-1}) + \log(N/\delta))$ with $x \in (0, 1/2]$ is a convex function over $(0, 1/2]$, and for any $\Delta \in [a,b]$, we have
\begin{align*}
    F({\Delta})
    -
    F(b)
    & \le 
    \frac{F(a)
    -
    F(b)}
    { b
    - a }
    \cdot 
    (b - \Delta)
    =
    L(a,b)
    \cdot 
    (b - \Delta).
\end{align*}

Furthermore, under the event of Lemma~\ref{lemma:accuracy-gap}, for any $z \in [N]$ such that $S_z > 2M^2 \log(NM/\delta)$, we have $\bar{\Delta}_z \in [c\Delta_{u^*}^3 ,\Delta_{u^*}]$ and therefore
\begin{align*}
    F({\bar{\Delta}_z})
    -
    F(\Delta_{u^*})
    & \le 
    \frac{F(c\Delta_{u^*}^3)
    -
    F(\Delta_{u^*})}
    { \Delta_{u^*}
    - c\Delta_{u^*}^3 }
    \cdot 
    (\Delta_{u^*} - \bar{\Delta}_z)
    =
    L(\cU_0)
    \cdot 
    (\Delta_{u^*} - \bar{\Delta}_z).
\end{align*}
Here we use $ L(\cU_0) = \frac{F(c\Delta_{u^*}^3)
    -
    F(\Delta_{u^*})}
    { \Delta_{u^*}
    - c\Delta_{u^*}^3 }$ is indeed a instance-dependent factor, with only logarithmic dependent in $N$ and $\delta^{-1}$(in $F$). $c$ is a global constant and in fact $c = 1/25$.
\end{lemma}
\begin{proof}
Differentiate $F(x)$ twice and it can be verified that $F''(x) > 0$. For any $\Delta \in [a ,b]$, the inequality above is easy to prove via convexity.

The rest is to prove that $\forall t \in [N]$, we have $\bar{\Delta}_z \in [\Delta_{u^*}/M, \Delta_{u^*}]$. It is clear that the upper bound holds because $\Delta_{u^*} := \max_{u \in \cU_0} \Delta_u$.

The lower bound is proved as follows:
We still have $\bar{\Delta}_z > \Delta_{u^*}/M$ because at any time $u^*$ always remains in the user set and by the assumption $\Delta_u > 0$.

Also, since $S_z > 2M^2 \log(NM/\delta)$, by Lemma~\ref{lemma:accuracy-gap}, we have
\begin{align*}
    \Delta_{u^*}
    -
    \bar{\Delta}_z 
    & \le 
    4\sqrt{\frac{M \log(2MN/\delta)}{S_z}}
    \\
    & \le 
    4\sqrt{\frac{M \log(2MN/\delta)}{2M^2 \log(NM/\delta)}}
    \\
    & \le 
    \frac{4}{\sqrt{M}}.
\end{align*}
Now we will prove that 
\begin{align*}
    \max 
    \bigg\{
    \frac{\Delta_{u^*}}{M}
    ,
    \Delta_{u^*} - \frac{4}{\sqrt{M}}
    \bigg\}
    \ge 
    c \Delta_{u^*}^3.
\end{align*}
Suppose $\frac{\Delta_{u^*}}{M} < c \Delta_{u^*}^3$, then we have 
$M > c^{-1} \Delta_{u^*}^{-2}$, this means
\begin{align*}
    \Delta_{u^*} - \frac{4}{\sqrt{M}}
    \ge 
    \Delta_{u^*}
    -
    4 \sqrt{c} \Delta_{u^*}
    \ge 
    c \Delta_{u^*}^3.
\end{align*}
The last inequality is due to $\Delta_{u^*} \le 1/2$ and $c = 1/25$.
\end{proof}

Now we are ready to prove the main result:

\begin{proof} [Proof of Theorem~\ref{theorem:complexity-gap}]
Based on our algorithmic design, we will not eliminate any user until the cumulative number of queries $S_z$ reach the threshold $S_z \ge 2 M^2 \log(NM/\delta)$.  We have
\begin{align*}
    \mathrm{GAP}(N, M, \delta)
    & =
    \sum_{z=2}^{N}
    F(\bar{\Delta}_z)
    -
    F(\Delta_{u^*})
    \\
    & = 
    \underbrace{
    \sum_{z=2}^{N}
    \ind \{ S_z < 2M^2 \log(NM/\delta)\}
    \big(
    F(\bar{\Delta}_z)
    -
    F(\Delta_{u^*})
    \big)}_{I_1}
    \\
    & \qquad +
    \underbrace{
    \sum_{z=2}^{N}
    \ind \{ S_z \ge 2M^2 \log(NM/\delta)\}
    \big(
    F(\bar{\Delta}_z)
    -
    F(\Delta_{u^*})
    \big)}_{I_2}.
\end{align*}
For $I_1$, no elimination is performed, so $\cU_z = \cU_0$, and we have
\begin{align*}
    I_1 & =
    \sum_{z=2}^{N}
    \ind \{ S_z < 2M^2 \log(NM/\delta)\}
    \big(
    F(\bar{\Delta}_0)
    -
    F(\Delta_{u^*})
    \big).
\end{align*}

For each term in $I_2$, we have $F(\bar{\Delta}_z)
    -
    F(\Delta_{u^*})
    \le
    L(\cU_0) \cdot 
    4 \sqrt{\log(2MN/\delta)}
    \cdot \sqrt{\frac{M}{S_z}}$ due to Lemma~\ref{lemma:function-F} and Lemma~\ref{lemma:accuracy-gap}. Therefore,
\begin{align*}
    I_2
    & \le 
    L(\cU_0)
    4 \sqrt{\log(2MN/\delta)}
    \sum_{z=2}^{N}
    \ind 
    \{
    S_z \ge 
    2M^2 \log(NM/\delta)
    \}
    \sqrt{\frac{M}{S_z}}.
\end{align*}

\end{proof}

\subsection{Proof and Discussions of Proposition~\ref{prop:complexity-gap}}
Suppose $M = o(N^{1/2})$, since $S_z \ge z \log (z/\delta) \ge z$(at least one comparison for an item), from~\eqref{eqn:comp-gap} we have
\begin{align*}
    \sum_{z=2}^{N}
    \ind 
    \big\{
    S_z < 
    2M^2 \log(NM/\delta)
    \big\}
    \le 
    \sum_{z=2}^{N}
    \ind 
    \big\{
    z < 
    2M^2 \log(NM/\delta)
    \big\}
    =
    o(N).
\end{align*}
The third term can be bounded with the fact $\ind 
    \big\{
    z < 
    2M^2 \log(NM/\delta)
    \big\} \le 1$,
\begin{align*}
    & L(\cU_0)\sqrt{\log(2MN/\delta)}
    \sum_{z=2}^{N}
    \ind 
    \{
    S_z \ge 
    2M^2 \log(NM/\delta)
    \}
    \sqrt{\frac{M}{S_z}}
    \\
    & \qquad \le 
    L(\cU_0)\sqrt{\log(2MN/\delta)}
    \sum_{z=2}^{N}
    \sqrt{\frac{M}{S_z}}
    \\
    & \qquad \le 
    L(\cU_0)\sqrt{\log(2MN/\delta)}
    \sum_{z=2}^{N}
    \sqrt{\frac{M}{z}}
    \\
    & \qquad \le 
    2L(\cU_0)\sqrt{\log(2MN/\delta)}
    \sqrt{MN}
    \\
    & \qquad =
    O(L(\cU_0)\sqrt{\log(MN/\delta)}
    \sqrt{MN}).
\end{align*}
$L(\cU_0)$ is actually dominated by the minimal mean accuracy $\min_z \Bar{\Delta}_z$ throughout the algorithm. 
In practice, $L(\cU_0)$ is usually a constant, related to all users' accuracy. In the worst theoretical case, $L(\cU_0)$ will be dominated by $F(\Delta_{u^*} / M) = \Tilde{O}(M^2)$, which further turns the last term into $\Tilde{O}(M^{5/2} N^{1/2})$, and requires $M = o(N^{1/5})$ so that this term becomes negligible.

\bibliography{ref}
\bibliographystyle{ims}
\end{document}